\pdfoutput=1
\documentclass[letterpaper]{article} 
\usepackage{aaai2026}  
\usepackage{times}  
\usepackage{helvet}  
\usepackage{courier}  
\usepackage[hyphens]{url}  
\usepackage{graphicx} 
\usepackage{natbib}  
\usepackage{caption} 
\usepackage{algorithm}
\usepackage{algorithmic}

\usepackage{newfloat}
\usepackage{listings}
\DeclareCaptionStyle{ruled}{labelfont=normalfont,labelsep=colon,strut=off} 
\floatstyle{ruled}
\newfloat{listing}{tb}{lst}{}
\floatname{listing}{Listing}
\usepackage[utf8]{inputenc} 
\usepackage{booktabs}       
\usepackage{amsfonts}       
\usepackage{nicefrac}       
\usepackage{microtype}      
\usepackage{xcolor}         

\usepackage{subcaption}
\usepackage{enumitem}
\usepackage{adjustbox}
\usepackage{multirow}
\usepackage{latexsym}
\usepackage{amsmath, amssymb, amsthm}
\usepackage{colortbl}

\newtheorem{theorem}{Theorem}
\newtheorem{lemma}[theorem]{Lemma}
\newtheorem{remark}[theorem]{Remark}

\newcommand{\AlgName}{KeepKV}
\newcommand{\MergingProblemName}{Attention Sag}

\title{KeepKV: Achieving Periodic Lossless KV Cache Compression \\ for Efficient LLM Inference}

\author{
    Yuxuan Tian\textsuperscript{\rm 1},
    Zihan Wang\textsuperscript{\rm 1},
    Yebo Peng\textsuperscript{\rm 1},
    Aomufei Yuan\textsuperscript{\rm 1},
    Zhiming Wang\textsuperscript{\rm 1},\\
    Bairen Yi\textsuperscript{\rm 2},
    Xin Liu\textsuperscript{\rm 2},
    Yong Cui\textsuperscript{\rm 3},
    Tong Yang\textsuperscript{\rm 1}\thanks{Corresponding author.}
}
\affiliations{
    \textsuperscript{\rm 1}State Key Laboratory of Multimedia Information Processing, School of Computer Science, Peking University\\
    \textsuperscript{\rm 2}ByteDance\\
    \textsuperscript{\rm 3}Tsinghua University
}

\begin{document}

\maketitle

\begin{abstract}

Efficient inference of large language models (LLMs) is hindered by an ever-growing key-value (KV) cache, making KV cache compression a critical research direction. Traditional methods selectively evict less important KV cache entries, which leads to information loss and hallucinations. Recently, merging-based strategies have been explored to retain more information by merging KV pairs that would be discarded; however, these existing approaches inevitably introduce inconsistencies in attention distributions before and after merging, causing degraded generation quality. To overcome this challenge, we propose \AlgName{}, a novel adaptive KV cache merging method designed to preserve performance under strict memory constraints, achieving single-step lossless compression and providing error bounds for multi-step compression. \AlgName{} introduces the Electoral Votes mechanism that records merging history and adaptively adjusts attention scores. Moreover, it further leverages a novel Zero Inference-Perturbation Merging method, compensating for attention loss resulting from cache merging. Extensive experiments on various benchmarks and LLM architectures demonstrate that \AlgName{} substantially reduces memory usage while successfully retaining essential context information, achieving over $2\times$ inference throughput improvement and maintaining superior generation quality even with only $10\%$ KV cache budgets.

\end{abstract}

\section{Introduction} \label{sec:intro}

Transformer-based large language models (LLMs) have demonstrated remarkable capabilities across various applications \cite{touvron2023llama2,jiang2023mistral7b,openai2024gpt4report,wan2024llm_survey,rozière2024codellama}. To accelerate inference, LLMs commonly employ a key-value (KV) cache mechanism, which stores the KV embeddings of previously processed tokens to avoid redundant computations \cite{Vaswani2017AttentionisAllYouNeed,Dai2019TransformerXL,Rae2019CompressiveTF}. However, as LLMs continue to support increasingly longer context lengths, the size of the KV cache grows rapidly, becoming a major bottleneck for inference \cite{kwon2023PagedAttention}. For example, in the case of LLaMA-3-70B, a batch size of 128 with an 8K context length requires up to 320GB of KV cache memory \cite{grattafiori2024llama3}. Consequently, compressing the KV cache while preserving generation quality has become a crucial challenge.

Prior works mainly explore two approaches for KV cache compression: eviction-based and merging-based methods, both of which are inherently lossy. Eviction-based approaches selectively retain critical cache entries using heuristics like attention scores and token positions, permanently discarding less critical entries \cite{xiao2024StreamingLLM,zhang2023h2o,reid2024RoCo,liu2024scissorhands,li2024SnapKV,yang2024pyramidinfer} and thus causing context loss and potential hallucinations \cite{zhang2024cam}. In contrast, merging-based methods aim to integrate rather than discard KV entries to retain more information. Recent representative studies, such as CaM~\cite{zhang2024cam}, D2O~\cite{wan2024D2O}, and KVMerger~\cite{wang2024KVMerger}, have explored strategies like weighted key-value merging to mitigate context loss. Nevertheless, these methods vary widely in merge candidate selection and merging weight computation, and lack solid theoretical foundations. We observe that existing strategies inevitably induce attention inconsistencies and output perturbation. Specifically, the merged KV pair's attention score is lower than the sum of the original scores prior to merging—a phenomenon we term "\MergingProblemName{}" (illustrated in Figure~\ref{fig:Intro-algo_compare}). These issues underline the necessity for an efficient and theoretically grounded KV cache merging strategy.

\begin{figure*}[t]
  \centering
  \includegraphics[width=0.9\textwidth]{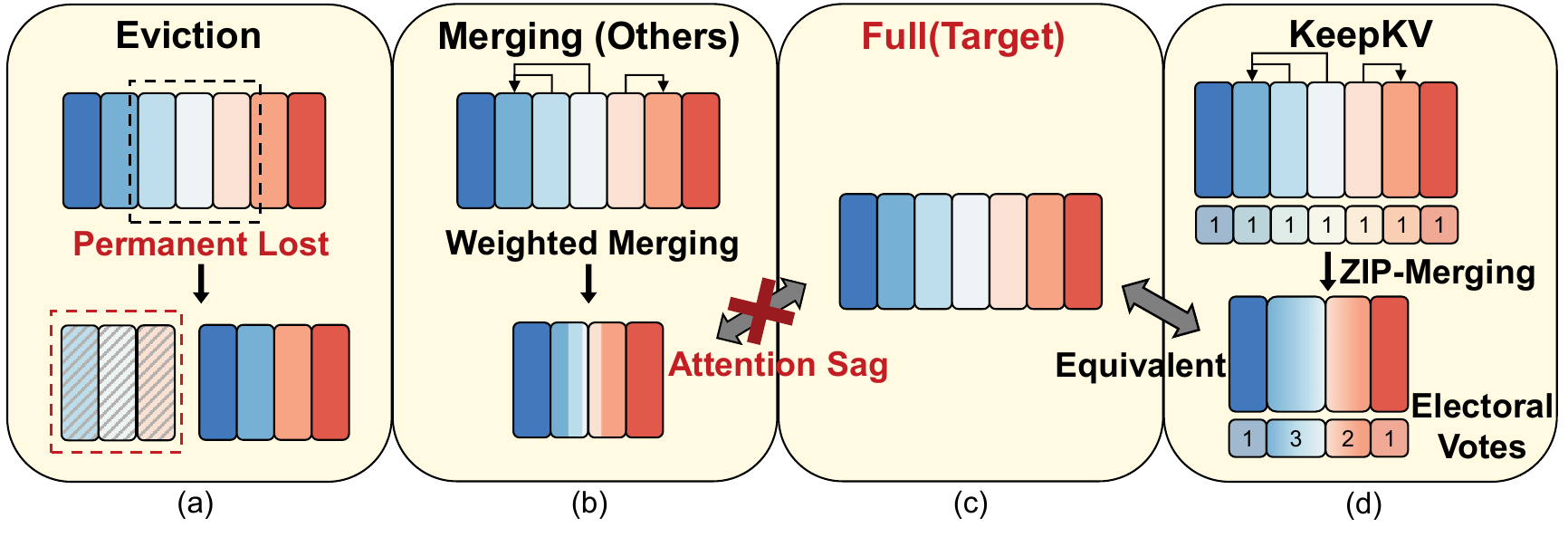} 
  \caption{Illustration of \AlgName{} vs. Existing Methods. The three middle blocks represent KV subject to eviction/merging. (a) Eviction methods permanently discard them. (b) Merging methods integrates them into retained KV, but the result is not equivalent to the full KV, causing "\MergingProblemName." (c) Full KV serves as the ideal baseline. (d) \AlgName{} uses Electoral Votes as merging records and applies ZIP-Merging to minimize output disturbance, ensuring consistency and improving performance.}
  \label{fig:Intro-algo_compare}
\end{figure*}

In this paper, we propose \AlgName{}, a novel KV cache merging method designed to maintain inference consistency and preserve essential contextual information. To the best of our knowledge, \AlgName{} is the first approach to achieve single-step lossless compression and to provide theoretical error bounds for multi-step compression. We first conduct a comprehensive theoretical analysis of existing eviction and weighted merging methods, grounded in the attention computation process, to reveal their fundamental limitations.Building on these theoretical insights, we propose a two-stage innovative design in \AlgName{}:

First, we propose techniques that achieve lossless compression for a single step. Specifically, we introduce the Electoral Votes mechanism, which records merging history, enabling accurate reconstruction of the original KV embeddings from compressed representations. Additionally, we present the Zero Inference-Perturbation Merging (ZIP-Merging) approach, which automatically adjusts weights to compensate for any losses caused by merging, maintaining attention consistency. These designs theoretically guarantee zero output perturbation at the current iteration despite compression.

Second, we extend \AlgName{} to multi-step generation by estimating attention scores based on historical patterns. This is motivated by our empirical observation of strong locality in attention scores, also confirmed in prior studies \cite{dong2024qaq,zhang2024cam}. Crucially, we provide theoretical analyses guaranteeing bounded output perturbation across multiple steps, thereby ensuring consistent inference quality under extended generation. Moreover, we offer a theoretical interpretation for prevalent similarity-based candidate selection methods, incorporating it into our design. 

By integrating these innovations, \AlgName{} further enables periodic lossless KV cache compression by storing a complete KV cache externally and periodically loading compressed representations into memory, thereby maintaining inference consistency and preserving essential contextual information. Through theoretical derivation and extensive experiments, we demonstrate that \AlgName{} effectively preserves attention stability and output consistency, outperforming state-of-the-art KV cache eviction and merging methods. The contributions of this paper are summarized as follows:

\begin{itemize}

\item We propose \AlgName{}, a novel adaptive KV cache merging approach, which introduces the Electoral Votes mechanism and Zero Inference-Perturbation Merging, achieving single-step lossless compression and providing error bounds for multi-step compression.

\item Extensive experiments across various tasks and models show that \AlgName{} maintains better performance under limited cache, outperforming existing KV cache eviction and merging methods. 

\item We are the first to theoretically analyze KV merging from the perspective of eliminating output perturbation. We provide guarantees on the perturbation bound of \AlgName{} and reveal the theoretical basis for merge candidate selection and weight design. Hopefully, our study can inspire future research on KV cache compression.

\end{itemize}

\section{Related Work} \label{sec:related_work}

KV cache has become a major bottleneck for efficient LLMs inference. Post-training optimization serves as a key solution due to its real-time and extensible capabilities.\cite{shi2024surveyonkvcache}. Existing methods fall into three categories: \textbf{quantization}, \textbf{eviction}, and \textbf{merging}.

\textbf{KV Cache Quantization.} Quantization methods convert tensor values to lower precision to reduce bit-width. KVQuant \cite{hooper2024KVQuant} applies Per-Channel Quantization for keys and Per-Token Quantization for values. MiKV \cite{yang2024MiKV} introduces mixed-precision KV caching, where less critical KV are stored at lower precision. Additionally, GEAR \cite{kang2024GEAR} leverages low-rank matrix approximation for quantization residuals to minimize quantization loss. Our \AlgName reduces the number of cached KV pairs through merging, which is orthogonal to quantization methods and can be combined for better efficiency.

\textbf{KV Cache Eviction.} Eviction methods only retain more important KV entries. StreamingLLM \cite{xiao2024StreamingLLM} and LM-infinite \cite{han2024lminfinite} identifies the importance of the initial k tokens for generation. H2O \cite{zhang2023h2o}, ScissorsHand \cite{liu2024scissorhands} and RoCo \cite{reid2024RoCo} recognize crucial KV based on attention scores, while SnapKV \cite{li2024SnapKV} utilizes attention within an observation window. Recent works explore improved budget allocation strategies across layers and heads. Pyramid \cite{cai2024pyramidkv,yang2024pyramidinfer} allocates more cache to lower layers, whereas AdaKV \cite{feng2024adakv}, HeadKV \cite{fu2024headkv}, and DuoAttention \cite{xiao2024duoattention} focus on inter-head differences. However, eviction causes irreversible information loss, potentially degrading generation quality.

\textbf{KV Cache Merging.} KV cache merging combines less important KV entries instead of permanently discarding them. DMC \cite{nawrot2024dmc} learns when and how to merge through training, which limits generalization and introduces extra overhead. In contrast, CaM \cite{zhang2024cam} adaptively merges evicted value states into others but does not merge the corresponding keys. Recently, D2O \cite{wan2024D2O} selects merge candidates and assigns merging weights based on cosine similarity between key states, while KVMerger \cite{wang2024KVMerger} introduces a clustering-based method to group merge candidates and computes merging weights using Gaussian Kernel Weights. However, these methods fail to maintain attention consistency before and after merging, leading to output perturbation. We propose a novel merging approach designed to eliminate output perturbation, supported by theoretical analysis and extensive comparisons.

\section{Methodology} \label{sec:meth}

\subsection{Preliminary: Inference with KV Cache} \label{sec:meth:pre}

We first introduce the attention computation process with KV cache. For simplicity, we consider a single attention head at one layer. Let the attention module's weight matrices be $W_q, W_k, W_v \in \mathbb{R}^{d \times d}$, where $d$ denotes the hidden dimension. During the prefill stage, given an input prompt tensor $X_L \in \mathbb{R}^{L \times d} = [x_1, x_2, \dots, x_L]$, where $L$ represents the prompt length, the KV states are computed and stored in the KV cache as follows:
\begin{align}
K_L &= X_L W_k = [k_1, k_2, \dots, k_L], \nonumber\\
V_L &= X_L W_v = [v_1, v_2, \dots, v_L].
\end{align}

In the decoding phase, KV cache are repeatedly utilized, while the newly computed KV pairs are continuously appended to it. Specifically, given the input at the $t$-th generation step, $x_t \in \mathbb{R}^d$, the KV cache update and attention computation are performed as follows:
\begin{align}
K_t &= [K_{t-1}, k_t], \quad k_t = x_t W_k \notag\\
V_t &= [V_{t-1}, v_t], \quad v_t = x_t W_v
\end{align}

\begin{align}
A^t &= \text{softmax} \left( \frac{q_t K_t^T}{\sqrt{d}} \right), \quad q_t = x_t W_q \notag\\
s^t_i &= e^{\frac{q_t k_i}{\sqrt{d}}}, \quad
o_t = \sum_{i=1}^{t}A_i^t v_i = \frac{\sum_{i=1}^{t}s^t_i v_i}{\sum_{i=1}^{t}s^t_i}
\label{eq:attention_compute}
\end{align}

KV cache effectively reduces redundant computation, but at the cost of increased memory consumption. Therefore, an important challenge is to compress the KV cache while maintaining model performance.

\subsection{Rethinking KV Cache Eviction and Merging} \label{sec:meth:rethinking}

Eviction and merging methods reduce memory usage by decreasing the number of stored KV pairs. The core motivation behind these studies is to minimize the impact of cache compression on the output. A fundamental \textit{subtask} is to ensure that the output $(o_t)$ remains as close as possible before and after compression at the current step. However, our analysis shows that existing methods inevitably introduce output perturbation and can not accomplish this task.

\textbf{Perturbation in KV Cache Eviction.} Eviction methods discard KV pairs deemed unimportant. Suppose we discard the pair $(k_e, v_e)$, and denote the output as $o'_t$. Based on Equation \ref{eq:attention_compute}, we obtain: 

\begin{equation}
o'_t = \frac{\sum_{i=1,i\neq e}^{t}s^t_iv_i}{\sum_{i=1,i\neq e}^{t}s^t_i} = \frac{1}{1-A^t_e}\left(o_t-A^t_ev_e\right).
\label{eq:eviction_perturbation}
\end{equation}

\begin{remark} 
Equation \ref{eq:eviction_perturbation} reveals that evicting \( (k_e, v_e) \) causes \( o'_t \) to deviate from \( o_t \), with the deviation primarily determined by \( A_e^t \). This formally explains why eviction methods generally prioritize discarding KV pairs with lower attention scores. 
\end{remark}

Although current methods optimize eviction and cache allocation strategies \cite{yang2024pyramidinfer,fu2024headkv} to minimize output impact, they cannot eliminate the perturbation in Equation \ref{eq:eviction_perturbation}. Previous studies have indicated that attention is not always sparse, especially in tasks requiring full context, as shown in Figure \ref{fig:information}. Moreover, evicted KV may become important later, but irreversible eviction leads to permanent loss.

\begin{figure*}[t]
    \centering
    \begin{subfigure}[t]{0.24\textwidth}
        \includegraphics[width=\linewidth]{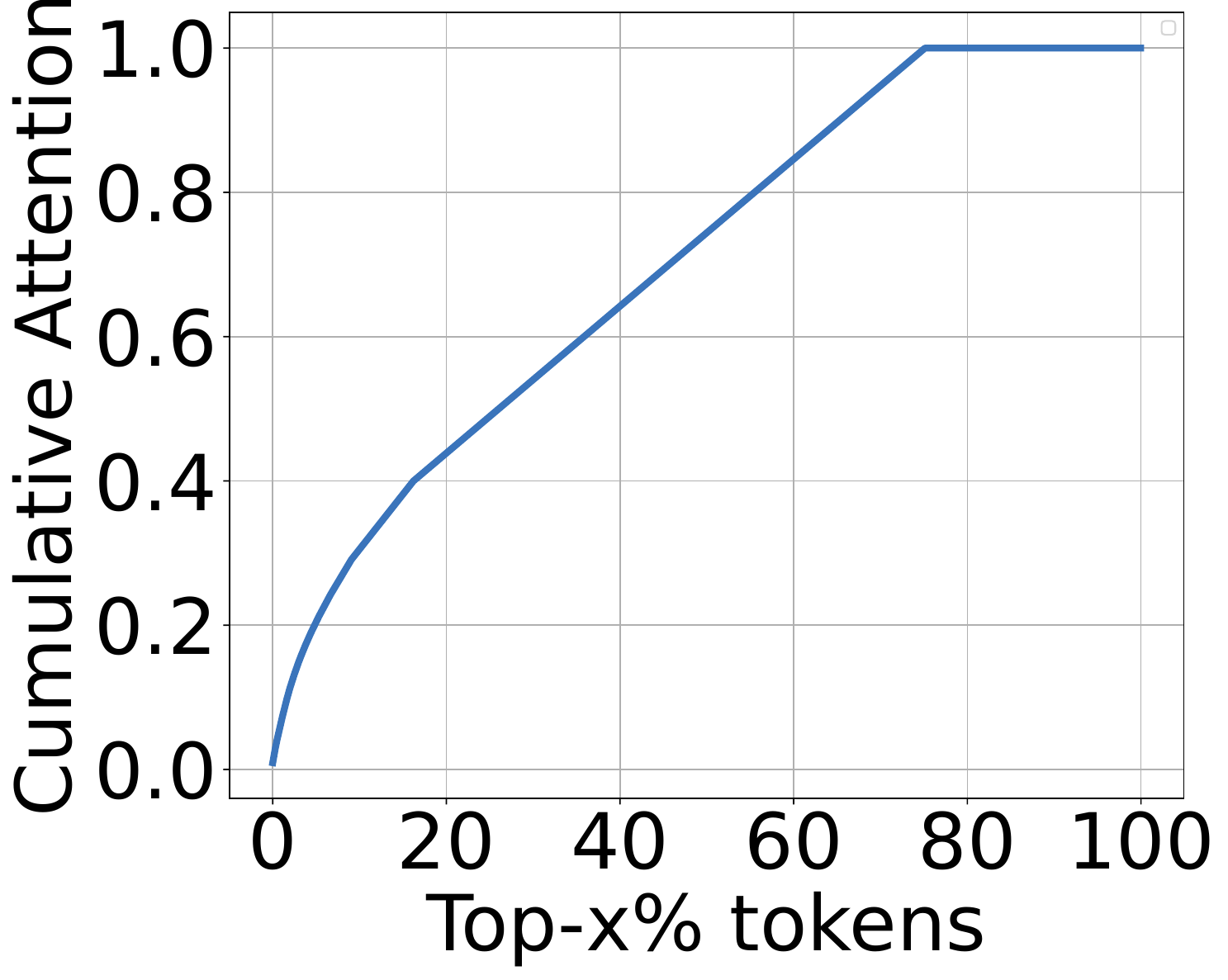}
        \caption{}
    \end{subfigure}
    \hfill
    \begin{subfigure}[t]{0.24\textwidth}
        \includegraphics[width=\linewidth]{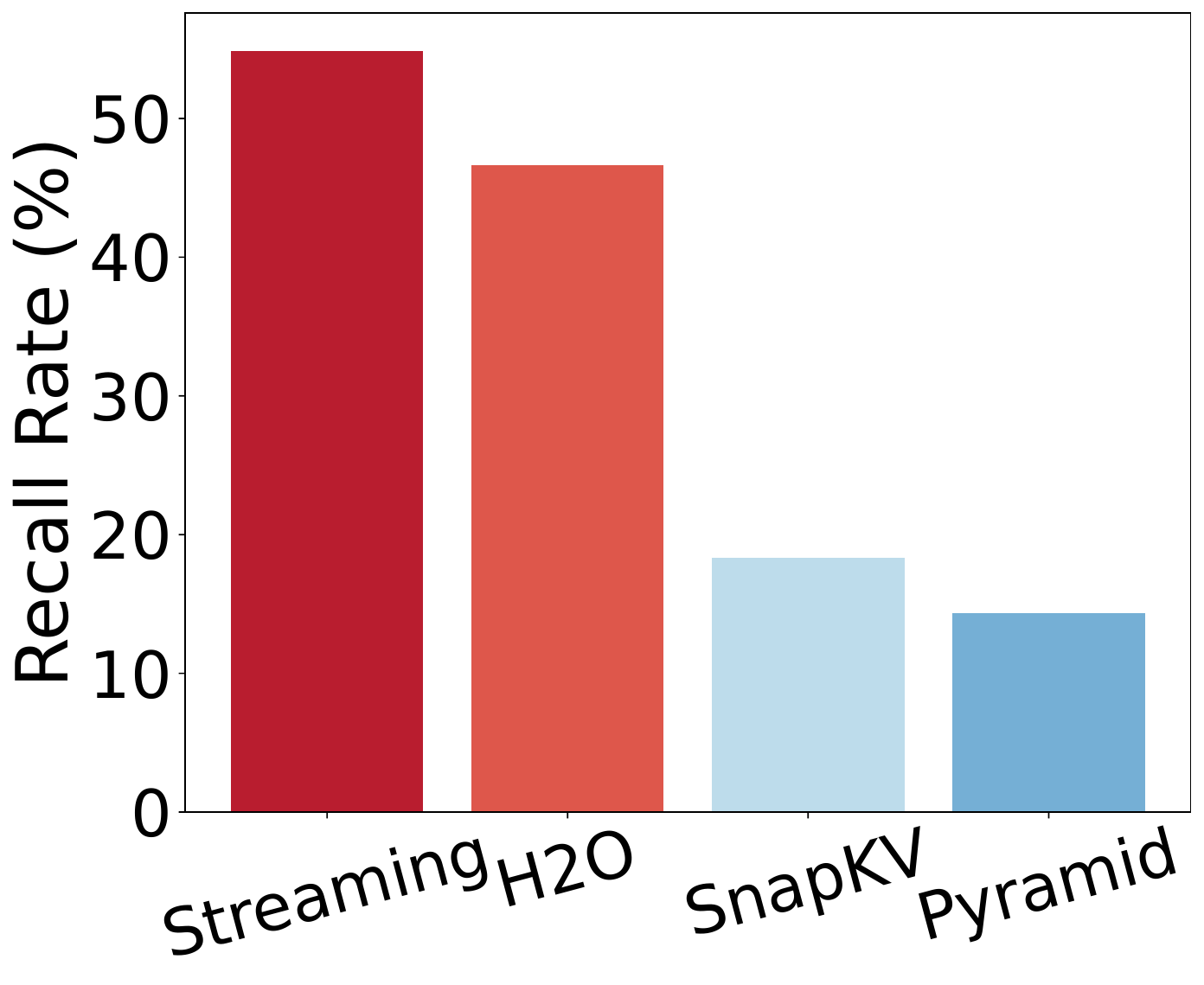}
        \caption{}
    \end{subfigure}
    \hfill
    \begin{subfigure}[t]{0.24\textwidth}
        \includegraphics[width=\linewidth]{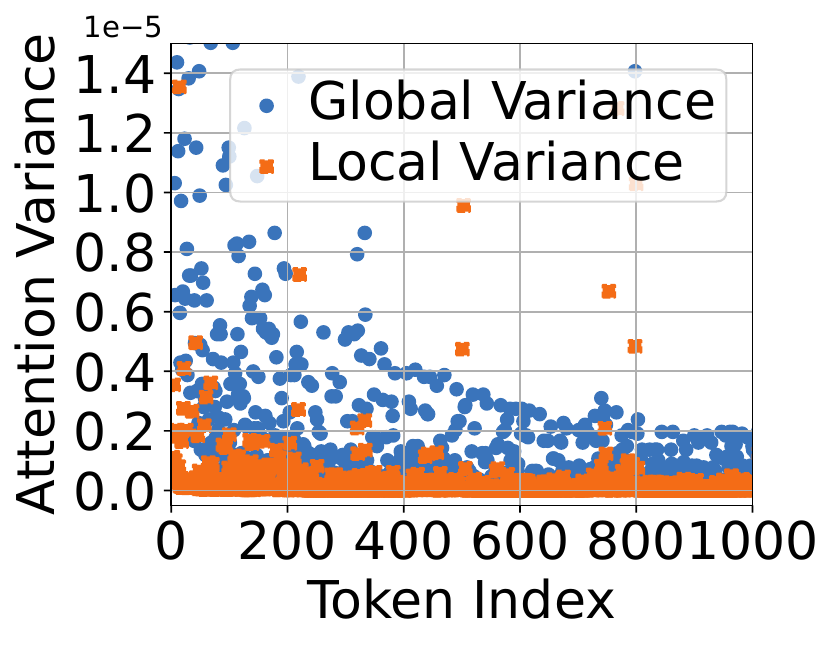}
        \caption{}
    \end{subfigure}
    \hfill
    \begin{subfigure}[t]{0.24\textwidth}
        \includegraphics[width=\linewidth]{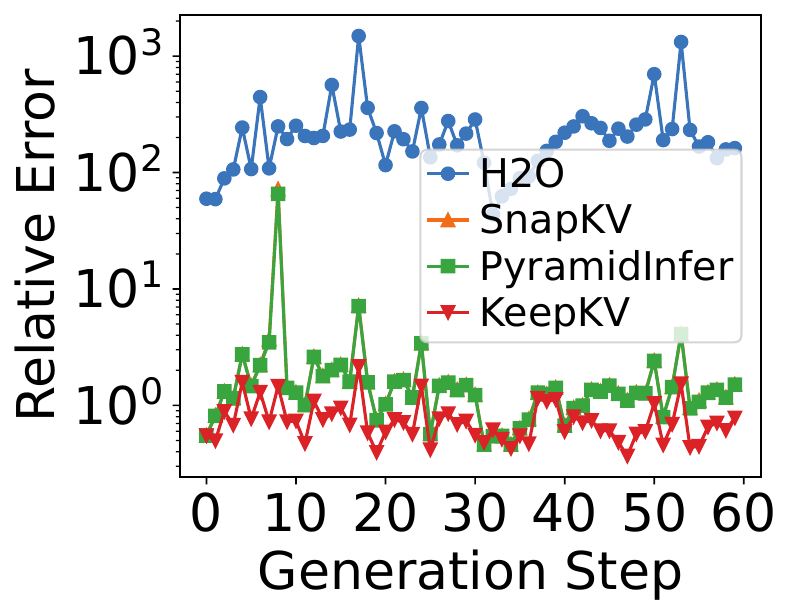}
        \caption{}
    \end{subfigure}
    \caption{(a) Cumulative distribution of attention scores. Retaining the top-$k$ tokens does not always preserve the majority of scores. (b) Proportion of to-be-evicted prompt tokens appearing in the top-20\% attention scores during generation (compression rate = 20\%). 
    (c) Each token's variance of its attention scores at each generation step (blue dots) is greater than the average variance within a sliding window (orange dots).
    (d) Relative errors for prediction of \AlgName{} and existing methods.
    }
    \label{fig:information}
\end{figure*}

\textbf{\MergingProblemName{} in KV Cache Merging.}  
Merging methods integrate less important KV into others rather than discarding them. Existing studies typically use weighted merging \cite{nawrot2024dmc,wan2024D2O,wang2024KVMerger}; formally, merging \( (k_e, v_e) \) into \( (k_c, v_c) \) is expressed as:

\begin{equation}
k_r = w_e k_e + w_c k_c,\quad v_r = w_e v_e + w_c v_c. 
\label{eq:weighted_merging}
\end{equation} 

Here, \( (k_r, v_r) \) are the merged vectors, with weights \( w_e, w_c \) determined by the merging method. In D2O \cite{wan2024D2O}, they depend on the cosine similarity between \( k_e \) and \( k_c \), while in KVMerger \cite{wang2024KVMerger}, they are computed using Gaussian Kernel values. The weights satisfy the normalization condition \( w_e + w_c = 1 \). However, this widely used convex combination method also introduces output perturbations: 

\begin{theorem}
Current weighted merging (convex combination) methods reduce the merged KV pair's attention score compared to the sum of the original scores before merging, i.e., ${A'}_r^t < A_e^t + A_c^t$, ultimately leading to $\left\| o'_t - o_t \right\| > 0$.
\label{thm:attn_collapse}
\end{theorem}

The formal proof is in Appendix \ref{sec:app:kv_merging}. We term this attention inconsistency from merging as \textbf{\MergingProblemName{}} and Figure \ref{fig:information} (c) illustrates this phenomenon. We provide an intuitive comprehension: existing methods merge multiple vectors into one, treating it equivalently as any other single vector in subsequent attention computations. This erases merging history, making it impossible to distinguish whether a KV pair is original or has absorbed numerous others. 

\subsection{Method: \AlgName{}} \label{sec:meth:keepkv_method}

\subsubsection{Electoral Votes and ZIP-Merging} \label{sec:meth:keepkv:votes_and_zipmerging}

\textbf{Electoral Votes.} To address \MergingProblemName{}, we propose the Electoral Votes mechanism, which records the number of merges $p_i$ (initialized to $1$) each KV pair undergoes. A natural analogy is the Electoral College system, where electors hold votes proportional to their state's population rather than a uniform share. The attention score of each KV is then scaled by its votes to approximate the original multiple KV's influence before merging. For example, if a KV pair $(k_r, v_r)$ has a vote count of $ p_r = 3 $, it is equivalent to three identical and independent instances of $(k_r, v_r)$ participating in the attention computation. Formally, the outputs before ($o_t$) and after merging ($o'_t$) are defined as follows:

\begin{align}
o_{t} &= \frac{\sum_{i=1}^{t}p_i s_i^{t} v_i}{\sum_{i=1}^{t}p_i s_i^{t}}, \nonumber \\
o'_{t} &= \frac{\sum_{i=1, i \neq e,c}^{t} p_i s_i^{t} v_i + p_r s_r^{t} v_r}
{\sum_{i=1, i \neq e,c}^{t} p_i s_i^{t} + p_r s_r^{t}}, \nonumber \\
p_r &= p_e + p_c.
\label{eq:keepkv_attention_compute}
\end{align}

\textbf{Zero Inference-Perturbation Merging (ZIP-Merging).} The Electoral Votes mechanism enables the elimination of output perturbations. We define the merging equations and theorem as follows:

\begin{align}
k_r &= \frac{(w_e k_e + w_c k_c) \ln{\frac{w_e + w_c}{p_e + p_c}}}
{w_e \ln{s_e^t} + w_c \ln{s_c^t}}, \nonumber \\
v_r &= \frac{w_e v_e + w_c v_c}{w_e + w_c}, \nonumber \\
w_e &= p_e s_e^t,\quad w_c = p_c s_c^t.
\label{eq:keepkv_merge}
\end{align}

\begin{theorem}
The merging method in Equation \ref{eq:keepkv_merge} is perturbation-free, that is, $\left\| o'_t-o_t\right\| = 0$
\label{thm:keepkv_zero_perturbation}
\end{theorem}

\begin{remark}
The proof is in Appendix \ref{sec:app:keepkv_method}. Intuitively, our method ensures attention consistency by preserving historical information via Electoral Votes and applying proper scaling (ZIP-Merging) to $(k_r,v_r)$ instead of a convex combination.     
\end{remark}

This theorem confirms that our novel merging approach can eliminate output perturbations and complete the subtask introduced  at the beginning of this section. However, its applicability remains limited to the current iteration, and extending it to multi-step generation requires additional design. 

\subsubsection{Extending to Multi-Step Generation}
\label{sec:meth:keepkv:multistep}

\textbf{EMA Attention Scores.} For ZIP-Merging to be effective in real-world multi-step generation, a solid comprehension of attention score dynamics is essential. Fortunately, empirical observations show that attention scores exhibit strong locality (Figure \ref{fig:information} (d)), meaning a token's attention scores evolve smoothly across adjacent steps, which is also validated by prior studies \cite{yang2024pyramidinfer,zhang2024cam,dong2024qaq}. From this, we employ the Exponential Moving Average (EMA) \cite{hunter1986EMA, busbridge2023EMAscale} with bias correction, a widely used technique in time-series analysis, formulated as follows:

\begin{equation}
\hat{s^{t}}=\frac{1}{1-{\alpha}^{t}}S^{t},\; S^{t} =\left\{
    \begin{array}{lr}
    \sum\limits_{k=t-w}^{t}(1-\alpha){\alpha}^{t-k}s^{k} , t=L &  \\
    \alpha S^{t-1}+(1-\alpha)s^{t}, t>L& 
    \end{array}
\right.
\label{eq:keepkv_predict}
\end{equation}

Note that after the prefill stage, we compute EMA scores using a recent window of length $ w $ rather than the entire sequence to obtain a more accurate estimation \cite{li2024SnapKV, yang2024pyramidinfer}. We find that this method outperforms mainstream approaches, such as cumulative attention and sliding window averaging, in predicting attention scores.
Building on this, replacing all score $s_i^{t}$ in Equation \ref{eq:keepkv_merge} with our EMA scores $\hat{s_i^{t}}$  from Equation \ref{eq:keepkv_predict} successfully achieves the extension to multi-step generation. Consequently, the future output perturbation becomes estimable and controllable. We present the following theorem and lemma(proof in Appendix \ref{sec:app:error_analysis}):  

\begin{theorem} 
For the $t'$-th step, let $\left|1-\frac{\hat{s_i^{t'}}}{s_i^{t'}} \right| \leq \epsilon, \epsilon <1$, the output perturbation satisfies $\Theta_{t'} < \frac{2\epsilon (1+\epsilon)\gamma}{(1-\epsilon)^2}  $, provided that $\left\|v_i-v_j\right\|\leq\gamma,\forall i\in [t'], j\in \{e,c\}$.
\label{thm:KeepKV_RE}
\end{theorem}

\begin{lemma} 
As the prediction error $\epsilon$ decreases and the merged candidates become increasingly similar, the output perturbation reduces to zero. That is, when either $\epsilon=0$ or $ (k_e,v_e)=(k_c,v_c)$, we have: $\Theta_{t'} = 0$.
\label{lemma:RE_towards_0}
\end{lemma}

\begin{figure*}[t]
  \centering

  \begin{minipage}[c]{0.43\textwidth}
    \centering
    \includegraphics[width=\linewidth]{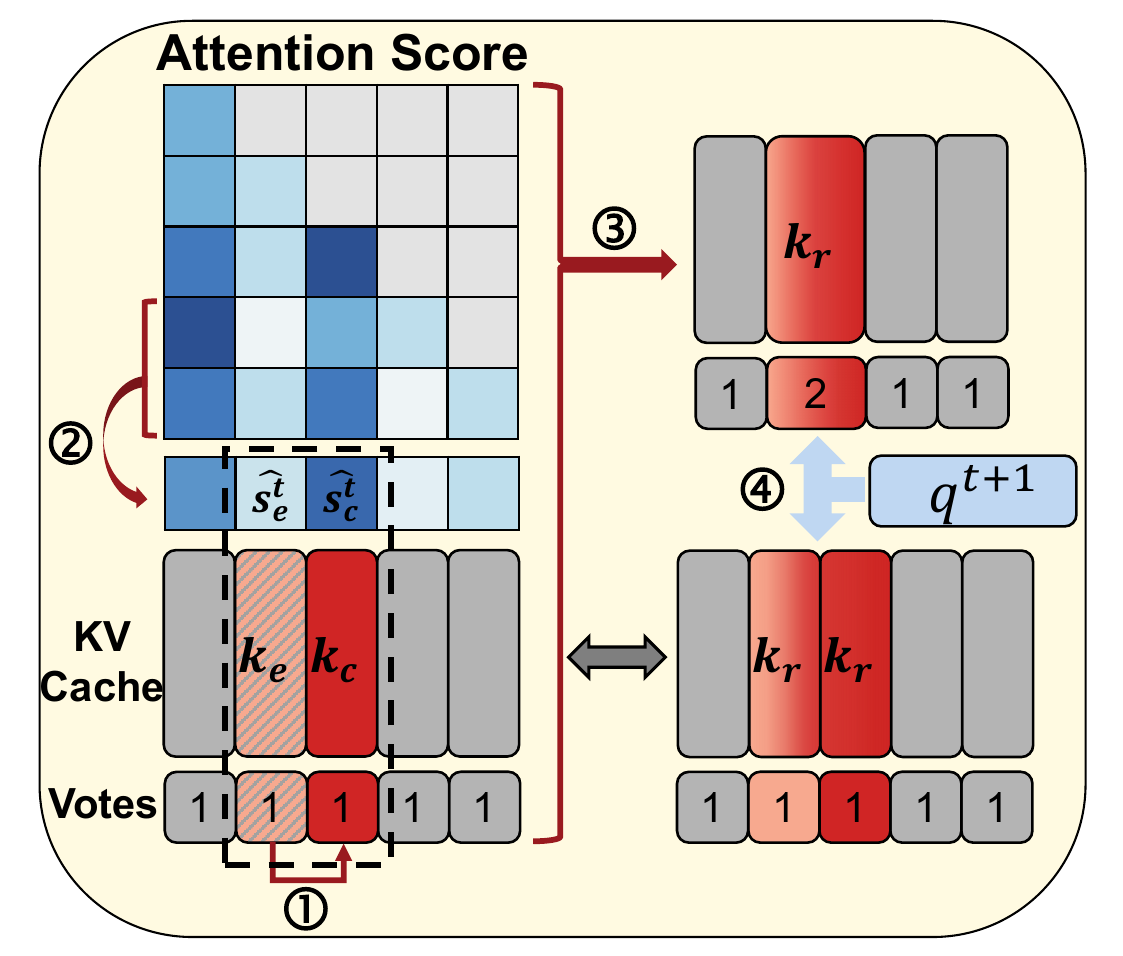}
  \end{minipage}%
  \hfill
  \begin{minipage}[c]{0.53\textwidth}
    \raggedright
    \hrule
    \textbf{Algorithm 1} KeepKV Merging at $t$-th Step\\[2pt]
    \hrule
    \begin{algorithmic}[1]
        \STATE \textbf{Input:} Attention scores $s^t$, EMA scores $S^{t-1}$, KV cache
        \STATE Let $K_e$ denote the to-be-evicted cache 
        \STATE Let $K_c$ denote the retained cache
        \STATE $U = \text{cosineSimilarity}(K_e,K_c)$
        \STATE For each $k_e \in K_e$, select $k_c = \operatorname*{Argmax}_{k_c\in K_c}(U^{e}),\ U_{e,c}>T$
        \STATE $\hat{s^{t}} = \text{updateEMA}(S^{t-1},s^t)$ \hfill $\triangleright$ Eq.~(8)
        \STATE \textbf{Merge:}
        \STATE \quad $k_r = \frac{\ln{((p_{e}\hat{s_{e}^t}+p_{c}\hat{s_{c}^t})/(p_{e}+p_{c}))}}%
                           {p_{e}\hat{s_{e}^t}\ln{\hat{s_{e}^t}}+p_{c}\hat{s_{c}^t}\ln{\hat{s_{c}^t}}}%
                           (p_{e}\hat{s_{e}^t}k_{e}+p_{c}\hat{s_{c}^t}k_{c})$
        \STATE \quad $v_r = \frac{1}{p_{e}\hat{s_{e}^t}+p_{c}\hat{s_{c}^t}}%
                           (p_{e}\hat{s_{e}^t}v_{e}+p_{c}\hat{s_{c}^t}v_{c})$
        \STATE \quad $p_r = p_e+p_c$
        \STATE Discard $(k_e,v_e), (k_c,v_c)$ and insert $(k_r,v_r)$ into KV cache
        \STATE \textbf{Output:} Updated KV cache
    \end{algorithmic}
    \hrule
  \end{minipage}

  \caption{Illustrative example of \AlgName{}. (0) $(k_e, v_e)$ is selected for eviction
  by specific compression method. (1) The retained KV with the highest cosine similarity,
  $(k_c, v_c)$, is selected. (2) EMA attention scores are updated. (3) ZIP-Merging is performed.
  (4) Consequently, with the Electoral Votes, the compressed KV can preserve the influence
  of the original KV in attention computations.}
  \label{fig:Meth-algo_process}
\end{figure*}

\textbf{Similarity-driven merging.} Lemma \ref{lemma:RE_towards_0} shows that output perturbation decreases as prediction error $\epsilon$ reduces, and closer merging objects result in lower perturbation. Clearly, if the merged KV pairs are identical, retaining one pair and setting its Electoral Votes to 2 introduces no error in subsequent computations. This provides a theoretical justification for prior merging strategies favoring high-similarity KV pairs \cite{wan2024D2O,wang2024KVMerger}. Following this, we merge each evicted KV pair with the retained one having the highest cosine similarity of keys, using a predefined threshold \( T \) to determine whether merging should occur, avoiding the overhead of dynamic adjustments like D2O \cite{wan2024D2O}. Furthermore, we observe that, during the prefill stage, reversing the conventional order—by first merging based on key similarity and then applying the eviction strategy, instead of merging after eviction as commonly done—can improve generation quality.

We present the workflow of \AlgName{} in Figure \ref{fig:Meth-algo_process}. Notably, \AlgName{} imposes no specific constraints on cache allocation or token selection strategies. It can directly integrate with common token selection methods by designating the merging pairs based on their eviction and retention sets, and it is also compatible with various cache allocation strategies. Thus, \AlgName{} demonstrates strong adaptability and can be combined with a range of mainstream cache compression methods, significantly enhancing both compression capability and generation quality.

\section{Experiment} \label{sec:exp}

\subsection{Experiment Settings} \label{sec:exp:setting}
\textbf{Tasks} We evaluate \AlgName{} on datasets with standard and extended context lengths, covering question-answering, summarization, and synthetic tasks. Specifically, for question-answering, we utilize MathQA \cite{amini2019mathqa}, OpenBookQA \cite{mihaylov2018openbookqa} and other tasks from the lm-eval-harness framework \cite{lmevalharness}. For summarization, we employ the XSUM\cite{narayan2018xsum} and CNN/DailyMail\cite{nallapati2016cnndm} tasks provided by the HELM framework. To assess performance on long-context tasks, we adopt LongBench\cite{bai2024longbench}, which effectively examines the algorithm's compression capabilities across diverse subtasks, including single-document QA, multi-document QA and synthetic tasks.

\textbf{Models and baselines.} Our evaluation is based on several representative LLMs, including Llama-2 \cite{touvron2023llama2}, Llama-3 \cite{grattafiori2024llama3}, and Mistral \cite{jiang2023mistral7b}. We compare our method against multiple baseline approaches: representative cache eviction methods such as Streaming \cite{xiao2024StreamingLLM}, H2O \cite{zhang2023h2o} and PyramidInfer \cite{yang2024pyramidinfer}, and prominent cache merging methods including CaM \cite{zhang2024cam} and D2O \cite{wan2024D2O}. For all baselines, we follow the authors’ released implementations and hyperparameter settings in their papers.

\textbf{Implementation.} In our main experiments, we set the merging threshold $T$ to $0.8$. For token selection and cache allocation, we follow the strategy recommended by PyramidInfer \cite{yang2024pyramidinfer}, which allocates fixed cache budgets, making it simple and efficient. And it is sufficient to demonstrate the advantages of our algorithm. In contrast, D2O \cite{wan2024D2O} applies dynamic allocation based on extra computation after prefill phase for each sequence. We implement \AlgName{} using the Hugging Face Transformers \cite{wolf2020huggingfacestransformers} and experiments are by default conducted on NVIDIA A100 80GB GPUs.

\begin{figure*}[t]
    \centering

    \begin{subfigure}[b]{0.24\textwidth}
        \includegraphics[width=\textwidth]{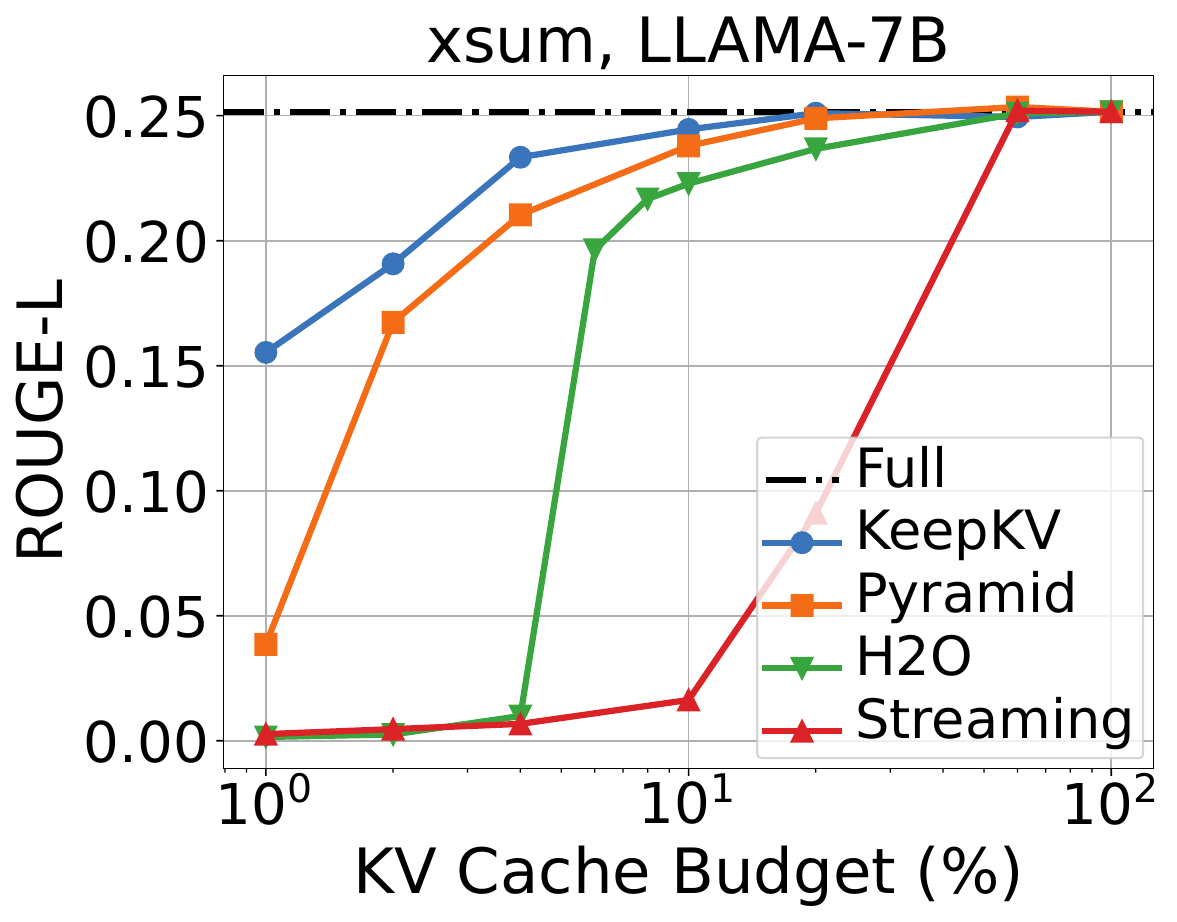}
    \end{subfigure}
    \begin{subfigure}[b]{0.24\textwidth}
        \includegraphics[width=\textwidth]{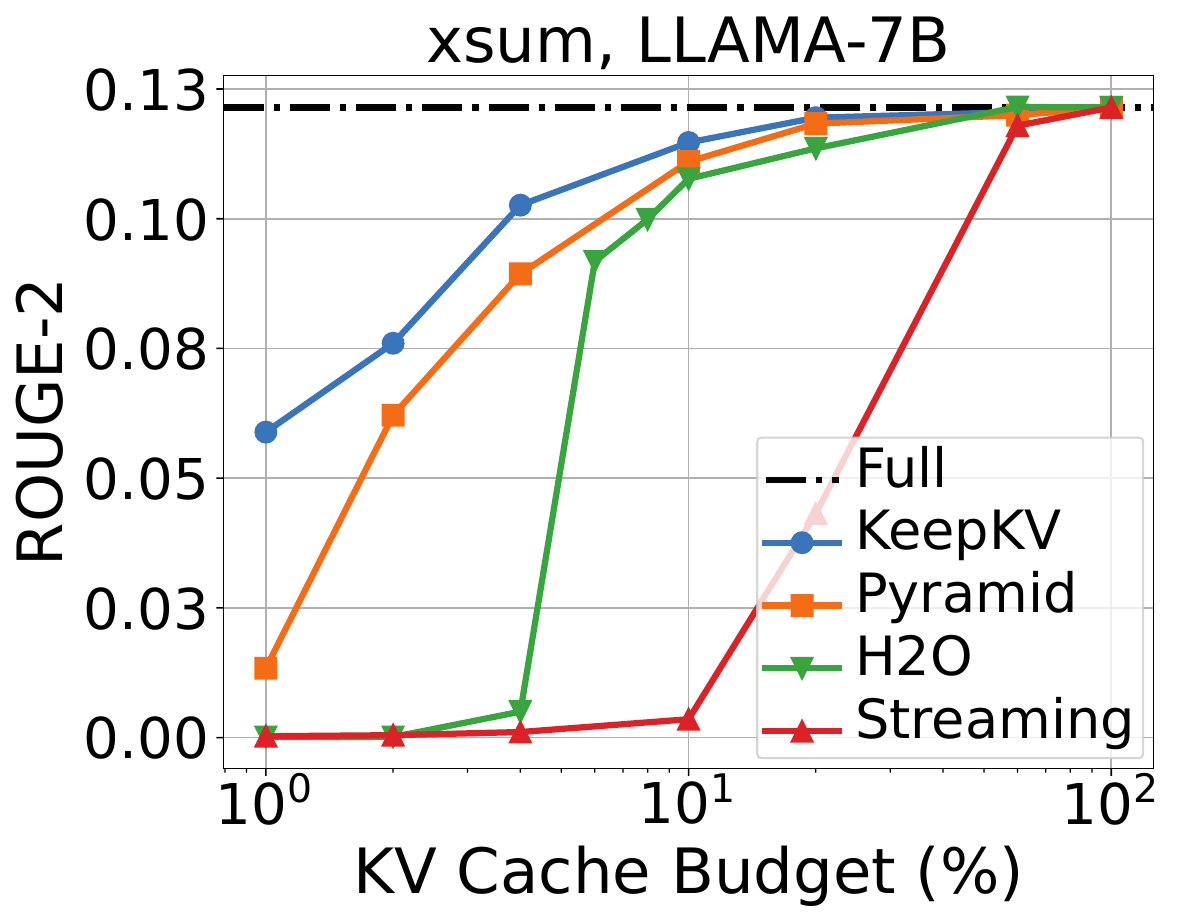}
    \end{subfigure}
    \begin{subfigure}[b]{0.24\textwidth}
        \includegraphics[width=\textwidth]{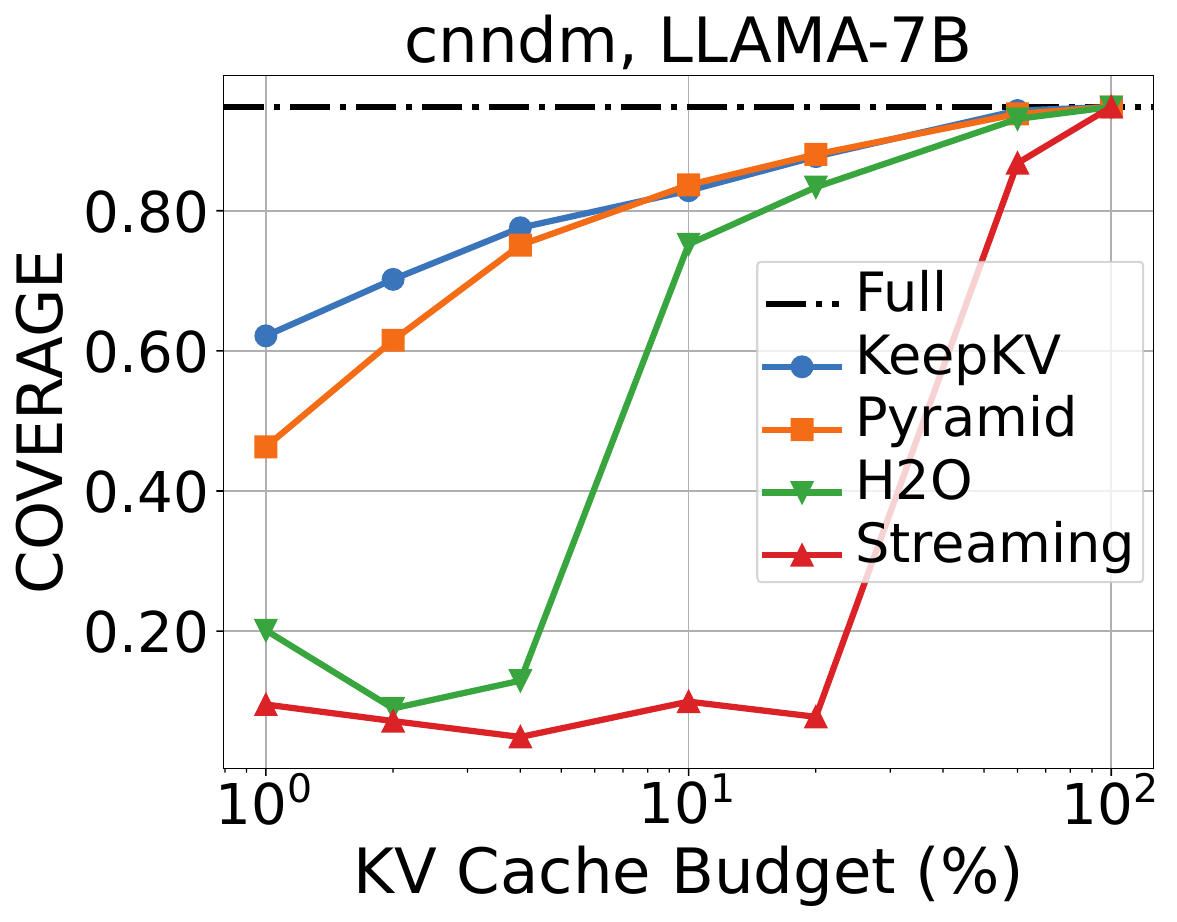}
    \end{subfigure}
    \begin{subfigure}[b]{0.24\textwidth}
        \includegraphics[width=\textwidth]{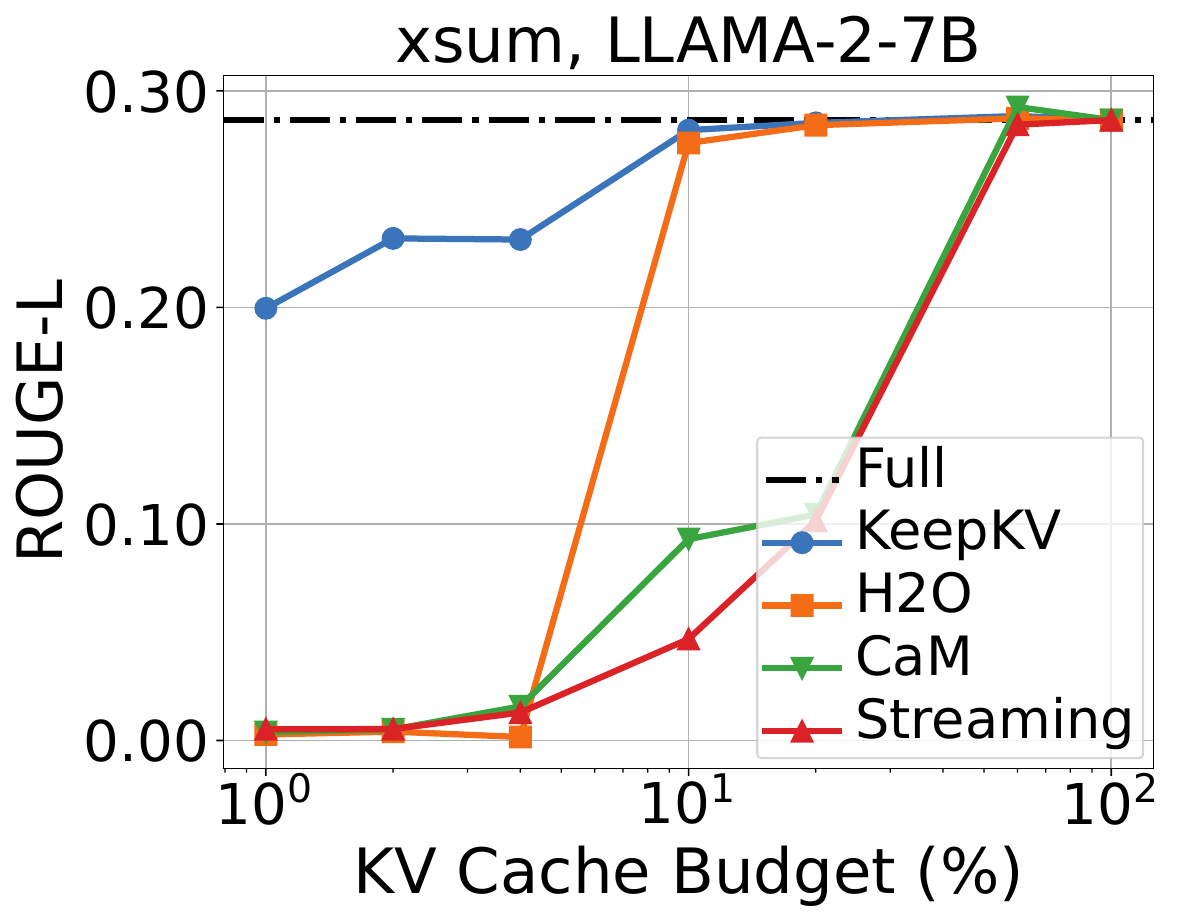}
    \end{subfigure}
    
    \begin{subfigure}[b]{0.24\textwidth}
        \includegraphics[width=\textwidth]{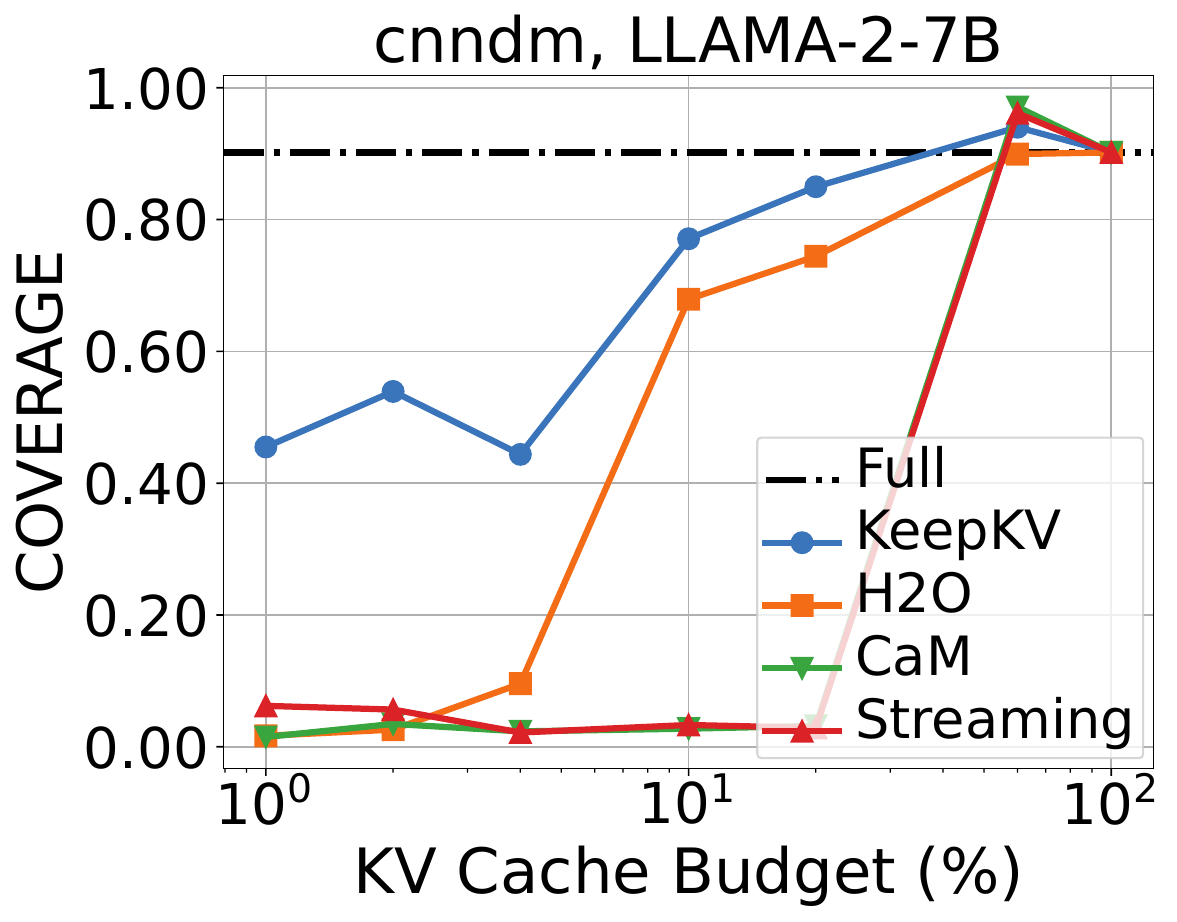}
    \end{subfigure}
    \begin{subfigure}[b]{0.24\textwidth}
        \includegraphics[width=\textwidth]{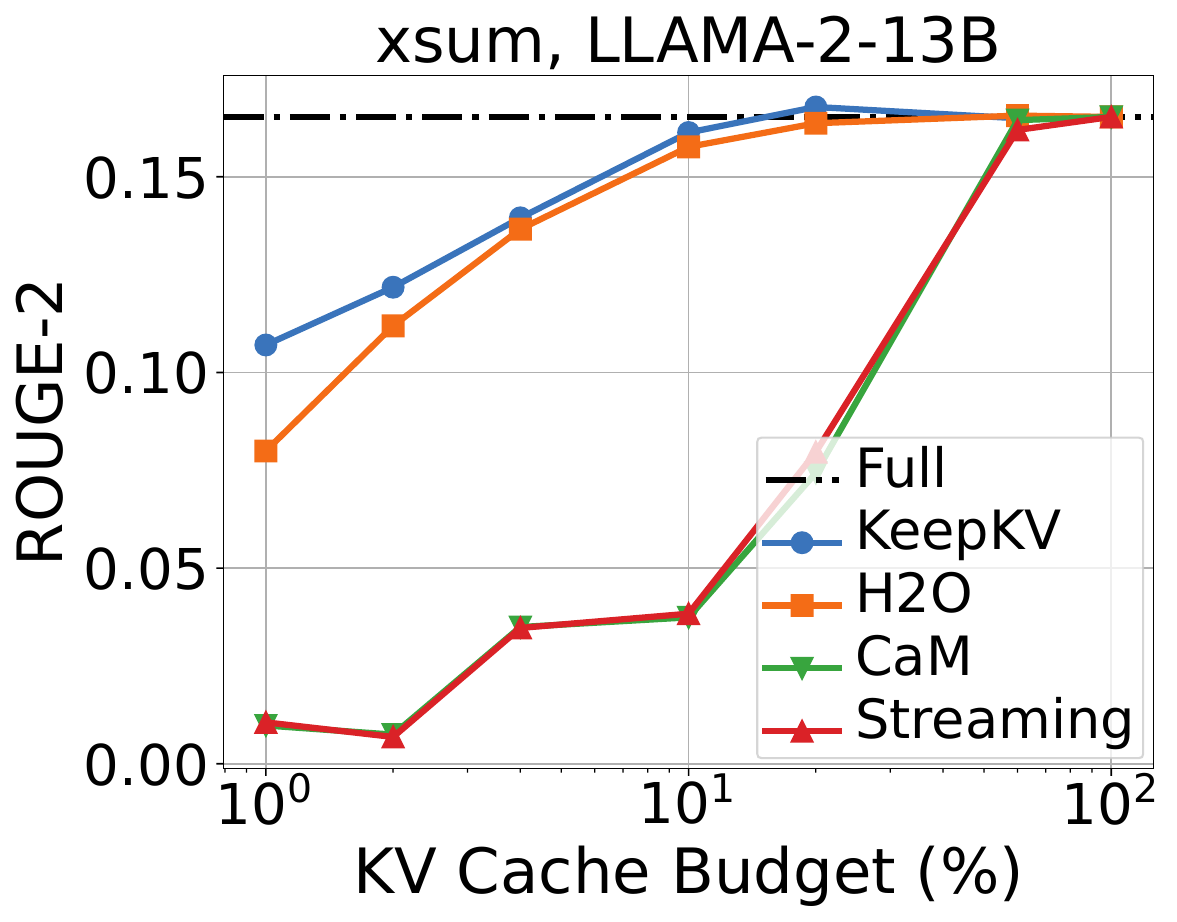}
    \end{subfigure}
    \begin{subfigure}[b]{0.24\textwidth}
        \includegraphics[width=\textwidth]{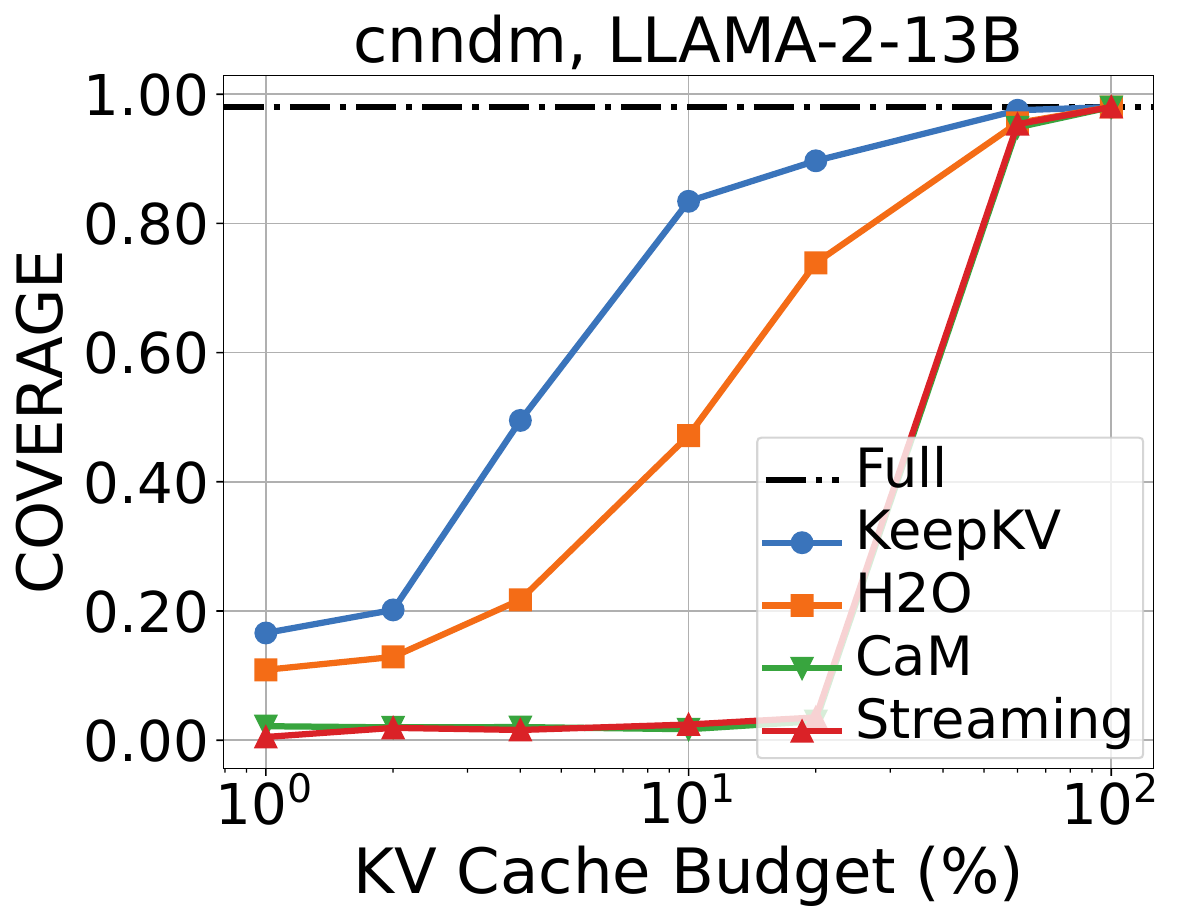}
    \end{subfigure}
    \begin{subfigure}[b]{0.24\textwidth}
        \includegraphics[width=\textwidth]{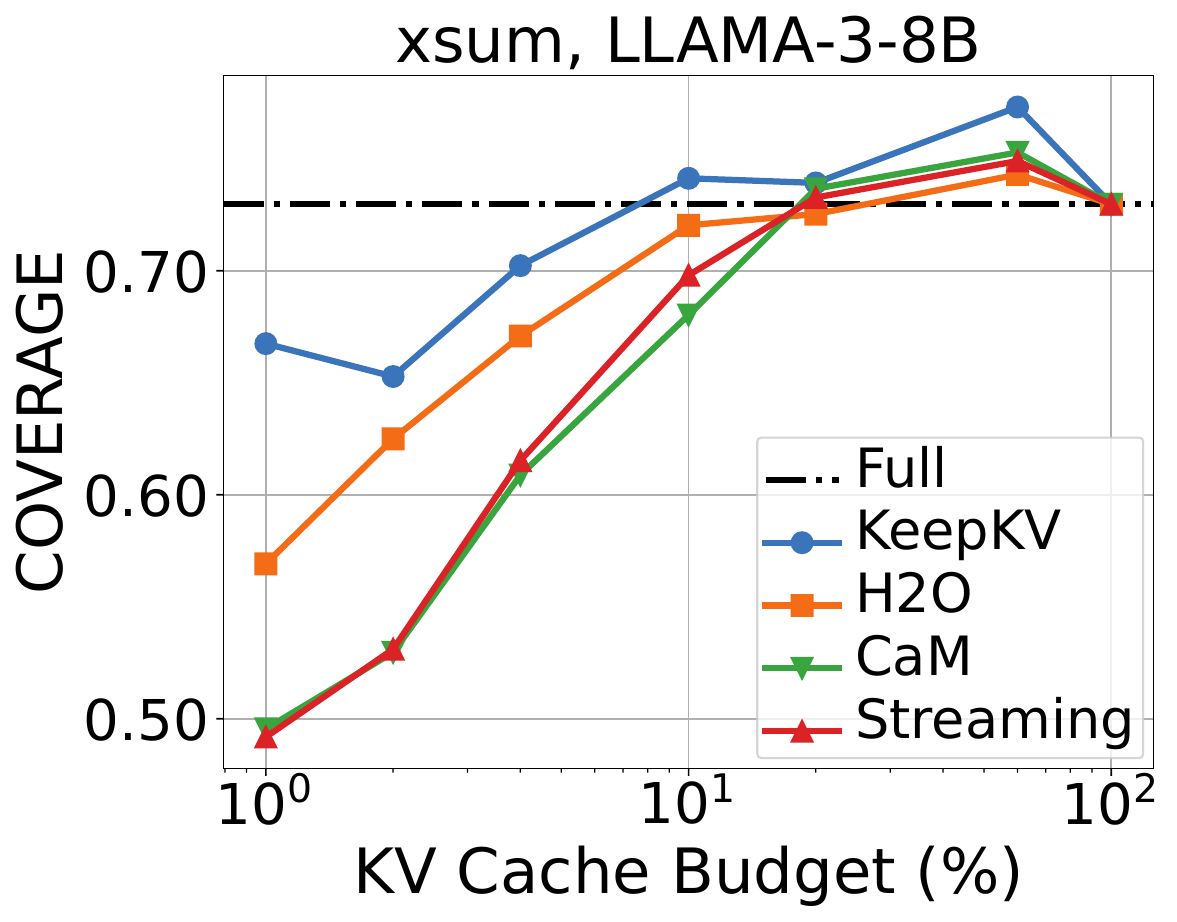}
    \end{subfigure}

    \begin{subfigure}[b]{0.24\textwidth}
        \includegraphics[width=\textwidth]{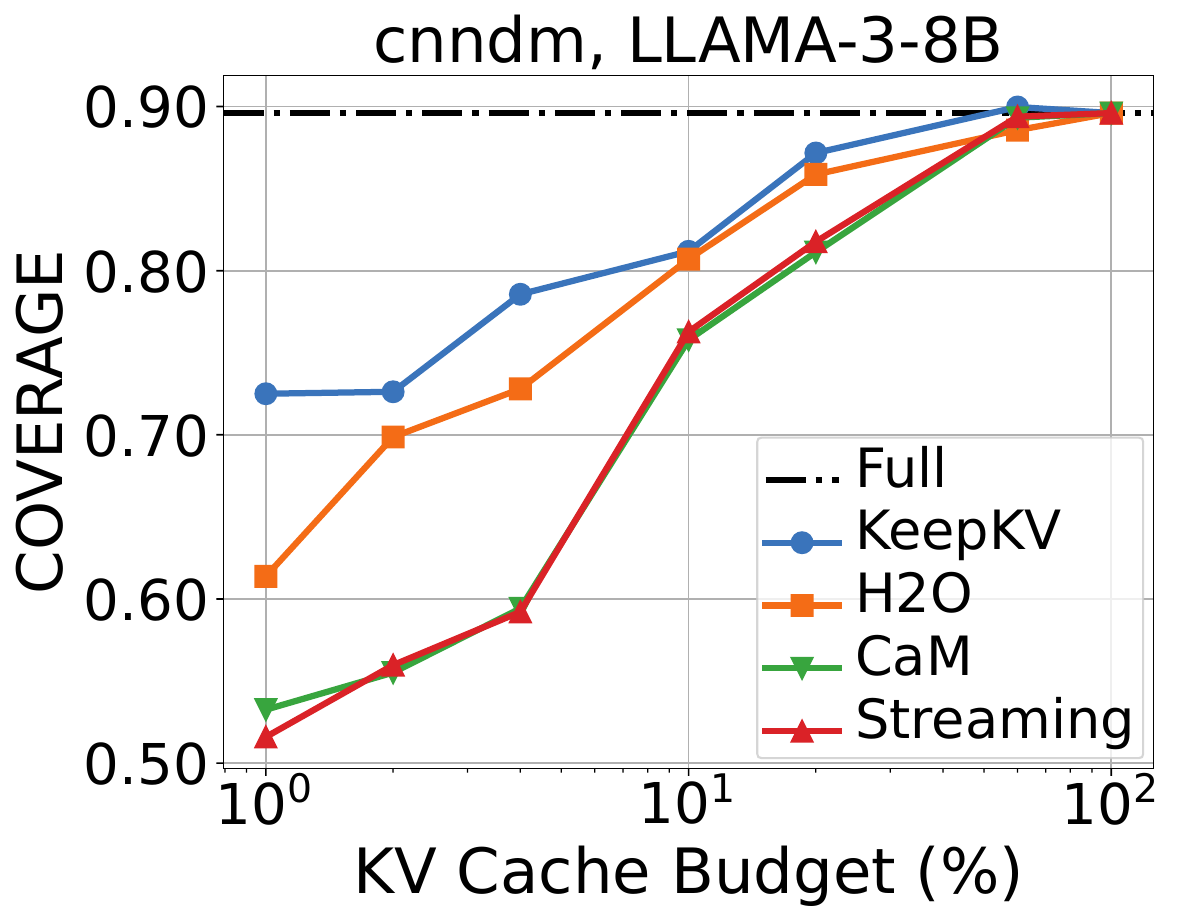}
    \end{subfigure}
    \begin{subfigure}[b]{0.24\textwidth}
        \includegraphics[width=\textwidth]{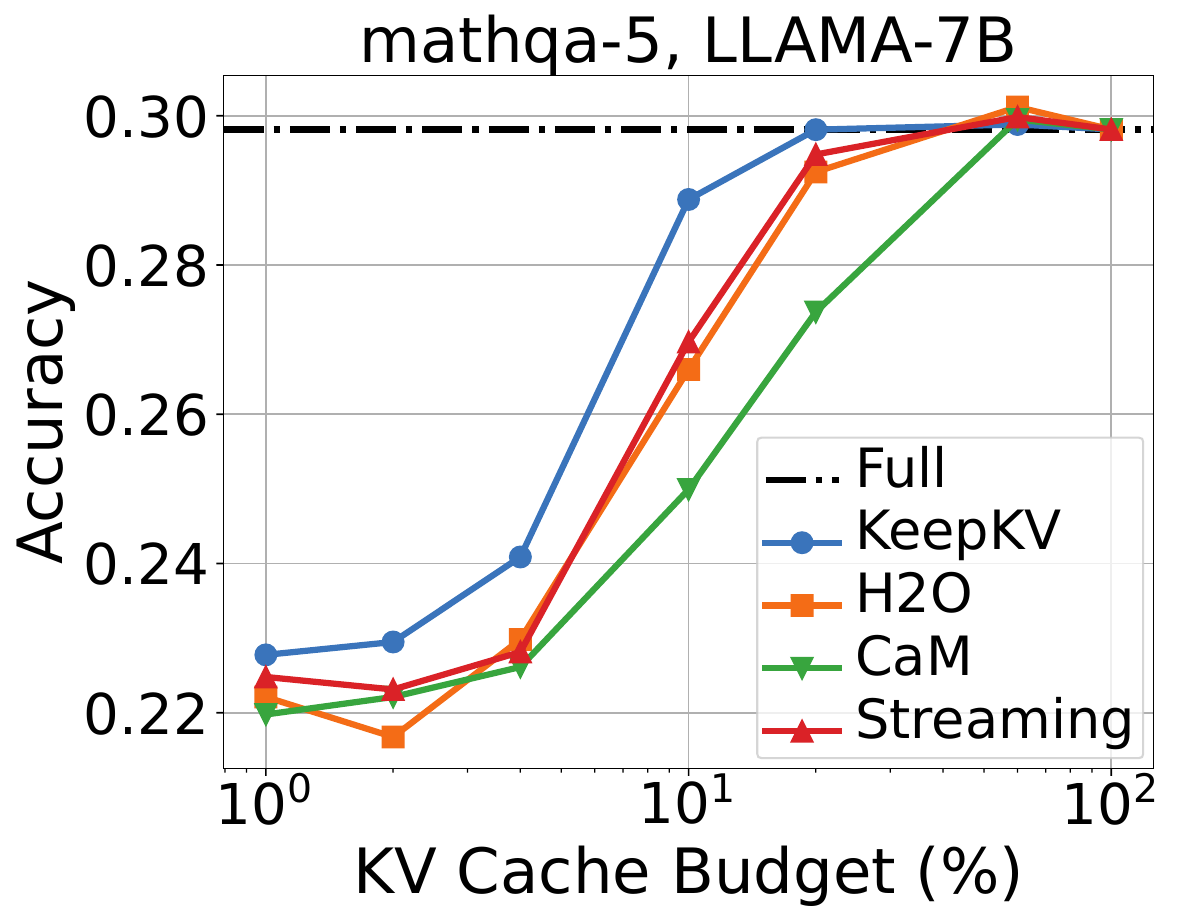}
    \end{subfigure}
    \begin{subfigure}[b]{0.24\textwidth}
        \includegraphics[width=\textwidth]{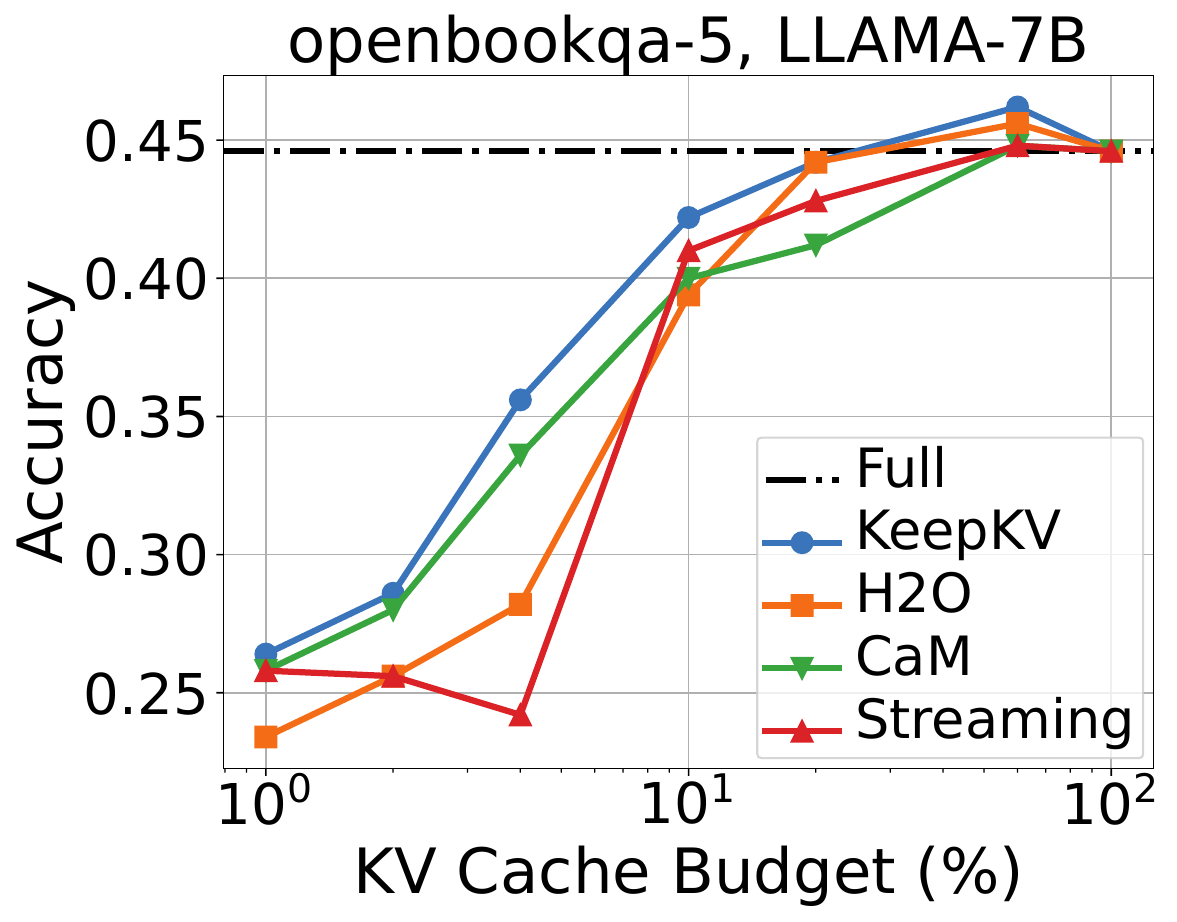}
    \end{subfigure}
    \begin{subfigure}[b]{0.24\textwidth}
        \includegraphics[width=\textwidth]{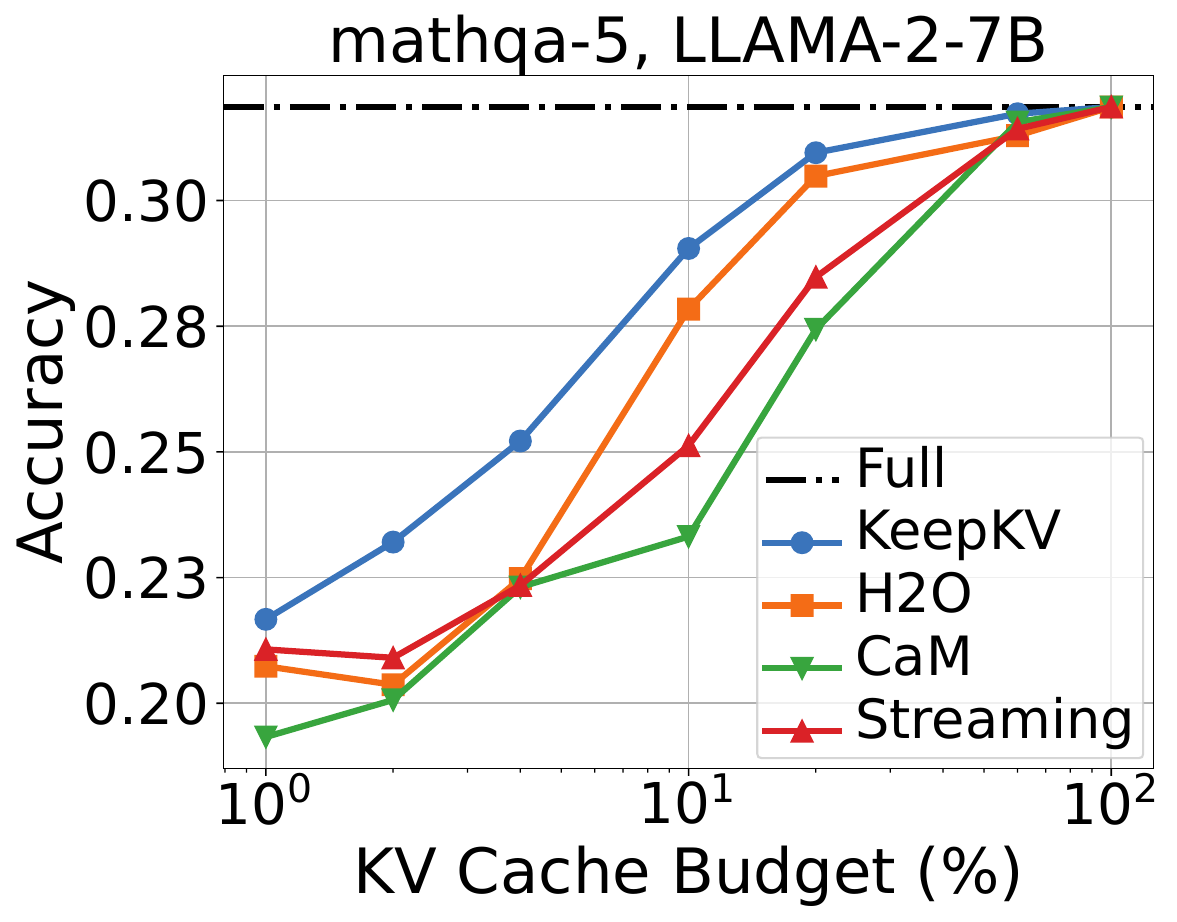}
    \end{subfigure}
    
  \caption{Performance of KeepKV and other methods for LLama backbones on HELM and LM-Eval evaluations.}
  \label{fig:Exper-compress_ratio_res}
\end{figure*}

\subsection{Accuracy on KV Cache Compression Ratios}

In Figure \ref{fig:Exper-compress_ratio_res}, we benchmark \AlgName{} on both the lm-eval-harness and HELM frameworks, comparing the fully cached KV version against multiple KV cache compression methods, including our proposed \AlgName{}. The x-axis represents the compression ratio, defined as the ratio between the compressed KV cache budget and the prompt length $L$. The results demonstrate that \AlgName{} consistently outperforms all other compression methods across various compression ratios. Particularly at extremely low compression rates, \AlgName{} achieves significantly better performance, highlighting its superior compression capability to retain maximal information within highly constrained memory budgets while effectively minimizing output perturbations introduced by compression.

\subsection{Accuracy on Long-context Tasks}

We evaluate \AlgName{} on the LongBench across Llama and Mistral model families,including Llama-2-7B, Llama-2-13B, Llama-3-8B and Mistral-7B, as shown in Table \ref{tab:app:longbench-res}. 
The evaluation tasks include Single-Document QA, Multi-Document QA, Summarization, Synthetic, and Code. The results indicate that \AlgName{} achieves performance closer to the full-cache baseline on most tasks, maintaining high generation quality despite limited cache availability. Notably, \AlgName{} significantly outperforms eviction-based methods, such as Local Window, StreamingLLM \cite{xiao2024StreamingLLM}, and H2O \cite{zhang2023h2o}. Furthermore, \AlgName{} also surpasses existing KV-cache merging methods, like CaM\cite{zhang2024cam} and D2O\cite{wan2024D2O}, underscoring the effectiveness of our carefully designed merging strategy in enhancing output accuracy.

\begin{table}[t]
\centering
\setlength{\tabcolsep}{4pt} 
\begin{tabular}{l|c|c}
\hline
Methods & Batch Size & Throughput (tokens/s) \\
\hline
Full cache & 2   & 116.54  \\
H2O        & 8   & 317.33  \\
D2O        & 8   & 214.8   \\
KeepKV     & 8   & 255.99  \\
\hline
\end{tabular}
\caption{Throughput comparison of \AlgName{} and other methods (4k context, 20\% compression ratio).}
\label{tab:app:throughput}
\end{table}

\definecolor{myblue}{RGB}{220,240,250}

\begin{table*}[t]
\centering
\small
\setlength{\tabcolsep}{2pt} 
\begin{tabular}{l|ccc|ccc|ccc|cc|cc}
\hline
\multirow{2}{*}{Methods} & 
\multicolumn{3}{c|}{Single-Doc QA} & 
\multicolumn{3}{c|}{Multi-Doc QA} & 
\multicolumn{3}{c|}{Summarization} & 
\multicolumn{2}{c|}{Synthetic} & 
\multicolumn{2}{c}{Code} \\
& NrtvQA & Qasper & MF-en & HotpotQA & 2WikiMQA & Musique & TREC & TriviaQA & SAMSum & PCount & PRe & Lcc & RB-P \\
\hline
\multicolumn{14}{c}{\textbf{Llama-2-7B}} \\
\midrule
\rowcolor{myblue} Full Model & 15.8 & 9.39	& 22.09	& 8.56 & 10.85	& 4.3 & 65.0 & 89.64 & 34.16	& 1.0 & 8.29 & 66.77 & 60.1 \\
Local Window & 2.22 & 9.29 & 1.83	& 5.14 & 7.18 & 1.02 & 17.5 & 4.07 & 3.17 & 1.5 & 2.58 & 16.31 & 15.35 \\
StreamingLLM & 11.81 & 5.18 & 19.26 & 7.07 & 10.48 & 3.71 & 55.5 & 87.31 & 31.84 & 1.5 & 4.29 & 63.79 & 56.07 \\
H$_2$O & 16.54 & 7.57 & 20.61 & 7.68 & 9.28 & 4.09 & \textbf{64.0} & 87.98 & 33.62 & 1.34 & 9.14 & 65.34 & 58.49 \\
CaM & 11.79 & 5.1 & 19.12 & 7.26 & \textbf{10.48} & 3.64 & 56.0 & 87.31 & 31.85 & 1.5 & 4.29 & 63.66 & 55.98 \\
D$_2$O & 16.04 & 6.54 & 19.48 & 8.14 & 10.12 & 4.62 & 63.5 & 88.39 & \textbf{34.1} & 1.39 & 7.54 & 65.8 & \textbf{59.44} \\
Ours & \textbf{17.32} & \textbf{7.48} & \textbf{22.2} & \textbf{8.51} & 9.72 & \textbf{4.65} & 60.5 & \textbf{88.87} & 33.2 & \textbf{2.23} & \textbf{8.45} & \textbf{65.9} & 56.36 \\
\midrule

\multicolumn{14}{c}{\textbf{Llama-2-13B}} \\
\midrule
\rowcolor{myblue} Full Model & 12.64 & 8.61 & 19.82 & 9.1 & 10.98 & 5.8 & 69.5 & 87.04 & 41.89 & 2.0 & 6.03 & 67.08 & 57.53 \\
Local & 4.95 & 5.11 & 3.82 & 7.05 & 9.87 & 3.42 & 19.0 & 7.83 & 2.63 & 1.17 & 6.51 & 16.7 & 14.65 \\
StreamingLLM & 5.04 & 5.75 & 12.24 & 9.4 & 10.47 & 4.71 & 57.0 & 82.48 & 37.21 & 1.5 & 5.04 & 61.47 & 50.84 \\
H2O & \textbf{13.83} & 6.41 & 15.52 & 9.04 & 9.55 & 5.53 & 66.0 & 86.08 & 40.2 & \textbf{2.88} & 7.37 & 64.52 & 55.46 \\
CaM & 5.16 & 5.95 & 12.31 & 9.19 & 10.52 & 4.66 & 57.0 & 82.48 & 37.28 & 2.5 & 5.25 & 61.75 & 50.71 \\
D2O & 12.76 & 6.53 & 14.87 & 8.59 & 10.34 & 5.75 & 66.5 & \textbf{86.52} & 40.52 & 2.0 & 6.99 & \textbf{65.23} & 55.84 \\
Ours & 12.09 & \textbf{6.89} & \textbf{17.81} & \textbf{9.49} & \textbf{10.54} & \textbf{5.79} & \textbf{66.8} & 82.72 & \textbf{41.35} & 1.75 & \textbf{7.55} & 64.81 & \textbf{56.29} \\
\midrule

\multicolumn{14}{c}{\textbf{Llama-3-8B}} \\
\midrule
\rowcolor{myblue} Full Model & 14.34 & 13.68 & 21.7 & 9.42 & 10.75 & 6.99 & 72 & 90.7 & 45.13 & 3.74 & 6.72 & 70.54 & 66.04 \\
Local & 2.14 & 6.69 & 5.17 & 6.16 & 5.0 & 2.42 & 34.25 & 30.5 & 10.66 & 2.36 & 2.0 & 28.91 & 24.52 \\
StreamingLLM & 10.43 & 7.84 & 13.85 & 9.18 & 10.44 & 5.47 & 61.0 & 90.37 & 44.35 & 2.6 & 10.5 & 68.49 & 63.94 \\
H2O & \textbf{13.73} & 10.02 & 17.2 & 9.31 & 10.62 & 6.42 & 63.3 & 90.44 & 45.02 & 3.29 & 7.56 & 68.95 & 63.84 \\
CaM & 10.43 & 7.83 & 13.89 & 9.11 & 10.37 & 5.47 & 61.0 & 90.37 & 44.31 & 3.16 & 10.5 & 68.59 & 64.04 \\
D2O & 13.5 & 8.86 & 17.21 & 9.16 & 10.52 & 6.35 & \textbf{65.5} & \textbf{90.52} & 44.64 & 3.44 & 5.8 & 68.49 & 64.84 \\
Ours & 12.76 & \textbf{10.63} & \textbf{18.57} & \textbf{9.37} & \textbf{10.72} & \textbf{6.53} & 64.5 & 90.33 & \textbf{45.2} & \textbf{3.54} & \textbf{7.16} & \textbf{69.05} & \textbf{65.68} \\
\midrule

\multicolumn{14}{c}{\textbf{Mistral-7B}} \\
\midrule
\rowcolor{myblue} Full Model & 22.92 & 39.74 & 51.46 & 43.28 & 39.46 & 25.59 & 74.0 & 88.64 & 46.97 & 4.0 & 63.5 & 61.42 & 58.72 \\
Local & 16.89 & 16.92 & 21.11 & 23.33 & 22.49 & 10.23 & 58.5 & 81.29 & 36.3 & 2.1 & \textbf{7.71} & 41.1 & 47.88 \\
StreamingLLM & 16.76 & 17.28 & 21.41 & 24.16 & 22.54 & 10.72 & 60.3 & 82.21 & 37.43 & 2.14 & 7.67 & 51.19 & 47.94 \\
H2O & 18.06 & 16.75 & 22.28 & 24.77 & 21.68 & 8.86 & 61.0 & 83.03 & 30.34 & 2.15 & 5.76 & 56.5 & 49.88 \\
CaM & 16.46 & 17.26 & 21.4 & 25.66 & 22.54 & \textbf{10.72} & 59.17 & 82.21 & 37.33 & 2.14 & 7.67 & 51.01 & 47.89 \\
D2O & \textbf{18.58} & 15.92 & 21.71 & 26.41 & 21.68 & 9.07 & 61.5 & 83.12 & 39.5 & 2.18 & 7.3 & 57.51 & 50.59 \\
Ours & 18.16 & \textbf{17.95} & \textbf{22.93} & \textbf{26.56} & \textbf{23.18} & 9.42 & \textbf{62} & \textbf{83.47} & \textbf{39.7} & \textbf{2.19} & 7.26 & \textbf{58.9} & \textbf{50.71} \\
\hline
\end{tabular}
\caption{Performance evaluation of \AlgName{} on various models in LongBench benchmarks (20\% compression ratio).}
\label{tab:app:longbench-res}
\end{table*}

\subsection{Throughput Analysis}

Our experiments demonstrate that \AlgName{} significantly enhances the inference throughput of the model by efficiently compressing the KV Cache, as illustrated in Table \ref{tab:app:throughput}. We conducted experiments on the Llama-2-7B model using an A100-80G GPU, with tasks derived from the LongBench evaluation framework. 
The experimental results indicate that various compression techniques improve throughput by reducing cache size and increasing batch size. Compared to the original full-cache method, KeepKV achieved over $2\times$ increase in throughput. It is noteworthy that, due to the additional computations, the throughput per request of merging methods is typically lower than that of classical eviction methods, such as H2O \cite{zhang2023h2o}. Nonetheless, \AlgName{} achieves higher throughput than the state-of-the-art (SOTA) merging-based algorithm, D2O \cite{wan2024D2O}. This advantage arises because D2O computes attention distribution variance for real-time cache allocation, whereas our method adopts a fixed cache allocation strategy to emphasizing the generalizability of \AlgName{}.

\subsection{Ablation Study}

To evaluate the generalizability of \AlgName{}, we conducted ablation experiments combining \AlgName{} with existing state-of-the-art eviction methods. Since \AlgName{} does not impose specific requirements on cache allocation or eviction/preservation strategies, it can be directly integrated with commonly used eviction/preservation policies. This only requires setting the merged parties as the eviction and preservation sets determined by their respective algorithms; it can also be applied with various cache allocation strategies, simply by modifying the cache configurations between layers and attention heads. As shown in Figure \ref{fig:ablation-evict}, we combined \AlgName{} with existing mainstream eviction methods, H2O\cite{zhang2023h2o} and PyramidInfer\cite{yang2024pyramidinfer}, using the HELM evaluation framework. The results demonstrate that, with \AlgName{} incorporated, the methods outperform the original ones across all compression ratios. This proves that our algorithm is highly scalable and versatile, capable of being integrated with various eviction schemes to enhance their compression efficiency and generation quality.

\begin{figure}[t]
    \centering
    \begin{subfigure}[t]{0.48\linewidth}
        \centering
        \includegraphics[width=\linewidth]{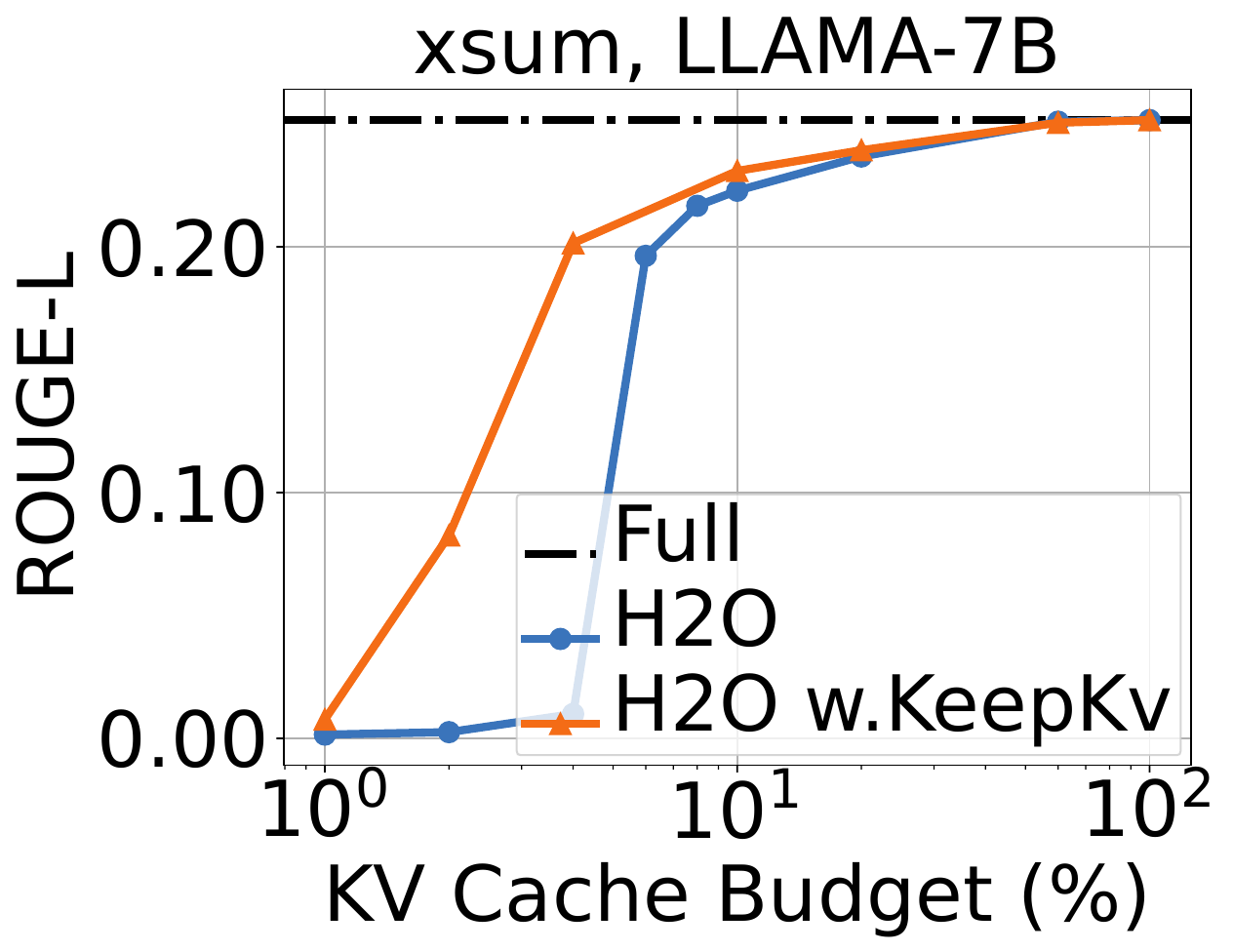}
        \caption{Combining with H2O.}
        \label{fig:h2owKeep}
    \end{subfigure}
    \hfill
    \begin{subfigure}[t]{0.48\linewidth}
        \centering
        \includegraphics[width=\linewidth]{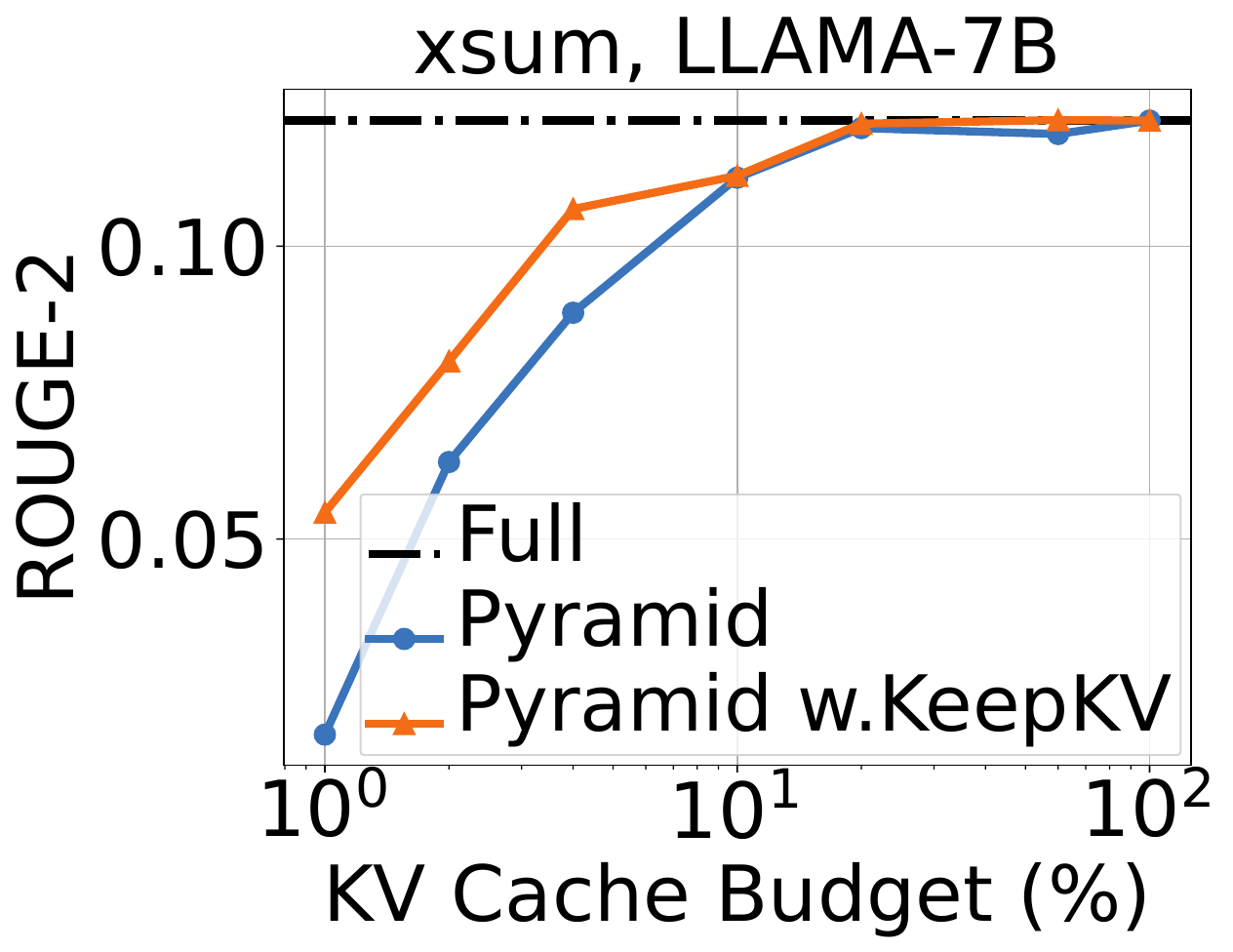}
        \caption{Combining with Pyramid.}
        \label{fig:pyrawKeep}
    \end{subfigure}
    \caption{Accuracy experiments combining \AlgName{} with existing eviction methods.}
    \label{fig:ablation-evict}
\end{figure}

\section{Conclusion} \label{sec:conclusion}

In this paper, we conduct a comprehensive analysis of the impact of KV cache compression on attention computation and propose \AlgName{}, which introduces the Electoral Votes mechanism and Zero Inference-Perturbation Merging to adaptively and dynamically merge the KV cache while minimizing output disturbance. \AlgName{} effectively preserves more information within limited memory, significantly mitigating the adverse effects of KV cache compression on generation quality. Our experiments demonstrate that \AlgName{} achieves performance closest to that of the full cache across various compression ratios. It also excels in both standard and long-context tasks. We believe \AlgName{} provides a novel perspective and a powerful tool for advancing KV cache compression methods, laying the foundation for efficient LLM inference.

\section*{Acknowledgements}
We are grateful to Chenhong He, Ruijie Miao, Yuhan Wu and Yanshu Wang from Peking University for their insightful discussions and helpful suggestions throughout the development of this research. We thank ByteDance Ltd. for providing technical support during the internship period.
This work was supported by the National Key Research and Development Program of China under Grant No. 2024YFB2906603, by the Beijing Natural Science Foundation (Grant No. QY25123), and in part by the National Natural Science Foundation of China (NSFC) (624B2005).

\bibliography{references}

\appendix

\section{Theoretical Analysis} \label{sec:app:thero}

Recently, many studies have analyzed KV cache compression strategies in LLM inference from a theoretical perspective \cite{zhang2023h2o,liu2024scissorhands,li2024SnapKV,yang2024pyramidinfer,zhang2024cam,wan2024D2O,wang2024KVMerger,gu2025attentionsink_empiricalview}. Overall, the primary objective of most existing works can be summarized as minimizing the impact of compression on the output. For instance, existing eviction-based methods and cache allocation strategies \cite{zhang2023h2o,liu2024scissorhands,reid2024RoCo,li2024SnapKV,yang2024pyramidinfer,cai2024pyramidkv,feng2024adakv,fu2024headkv,xiao2024duoattention} all aim to maximize the retention of essential information within limited memory by evicting less important tokens or reducing cache allocation in non-critical heads and layers based on empirical observations of attention distributions. However, eviction inevitably leads to irreversible information loss, which has motivated the development of KV cache merging methods \cite{zhang2024cam,wan2024D2O,wang2024KVMerger,nawrot2024dmc}. Despite this, key challenges such as the selection of merging candidates and the assignment of merging weights remain largely unexplored, with a lack of systematic theoretical foundations. In this work, we introduce a novel perspective distinct from prior approaches. We formulate the problem as eliminating output perturbation and derive a novel merging method by analyzing the attention computation process. First, we introduce Electoral Votes mechanism, making the elimination of output perturbation feasible. Then, we derive a merging computation formula to eliminate perturbation at the current step. Finally, we extend this framework to multi-step generation, providing a theoretical guarantee for output perturbation and offering a reasonable explanation for mainstream similarity-based merging candidates selection methods.

Specifically, we first demonstrate the unavoidable output perturbation caused by KV cache eviction. Next, we discuss the attention sag issue in existing KV cache merging methods and provide a formal proof. Then, we present the derivation process of our KeepKV merging method. Finally, we provide a theoretical guarantee for the output perturbation of KeepKV, including proofs for the main theorem and its associated lemma. The symbolic representation of the attention computation process remains consistent with the one introduced in the methodology section.

\subsection{Perturbation in KV Cache Eviction} \label{sec:app:kv_evcition}

Eviction methods discard KV pairs deemed unimportant. We denote the first generation step in the decoding phase as the $(L+1)$-th generation step, where $L$ represents the prompt length. And for a positive integer n, let $[n]:=\{1,2,...,n\}$. At $t$-th generation step, let $K_e = \{e_1 e_2,...,e_m\}, m\in [t]$ denoted the index of to-be-evicted cache. Based on Equation \ref{eq:attention_compute}, the output after eviction $(o'_t)$ is:

\begin{equation}
o'_t = \sum_{i=1,i\notin K_e}^{t}A_i^{'t}v_i,  \quad A_i^{'t} = \frac{s^t_i}{\sum_{i=1,i\notin K_e}^{t}s^t_i}.
\end{equation}

By transforming $o'_t$ towards $o_t$, we obtain:

\begin{equation}
\begin{aligned}
o'_t &= \frac{\sum_{i=1}^{t}s^t_i}{\sum_{i=1}^{t}s^t_i-\sum_{j\in K_e}^{t}s^t_j} * \frac{\sum_{i=1}^{t}s^t_iv_i - \sum_{j\in K_e}^{t}s^t_jv_j}{\sum_{i=1}^{t}s^t_i} \\
&= \frac{\sum_{i=1}^{t}s^t_i}{\sum_{i=1}^{t}s^t_i-\sum_{j\in K_e}^{t}s^t_j}\left(o_t-\frac{\sum_{j\in K_e}^{t}s^t_jv_j}{\sum_{i=1}^{t}s^t_i}\right) \\
&= \frac{\sum_{i=1}^{t}A^t_i}{\sum_{i=1}^{t}A^t_i-\sum_{j\in K_e}^{t}A^t_j}\left(o_t-\frac{\sum_{j\in K_e}^{t}A^t_jv_j}{\sum_{i=1}^{t}A^t_i}\right) \\
&= \frac{1}{1-\sum_{j\in K_e}^{t}A^t_j}\left(o_t-\sum_{j\in K_e}^{t}A^t_jv_j\right).
\label{eq:app:evict_problem}
\end{aligned}
\end{equation}

Equation \ref{eq:app:evict_problem} indicates that the difference between $o'_t$ and $o_t$ decreases as the attention score of the evicted KV $(A^t_j, j\in K_e)$ diminishes. When the total score of $K_e$ becomes negligible, the output perturbation at the current step approaches zero. This formally explains why eviction methods generally prioritize discarding KV pairs with lower attention scores. 

However, existing studies \cite{chen2024MagicPIG} have shown that attention can be relatively dispersed in certain tasks, meaning that evicting even a small number of tokens can have a non-negligible impact. Furthermore, as the compression ratio increases, evicted tokens will account for a significant portion of the attention scores, exacerbating the degradation of generation quality.

\subsection{Attention Sag in KV Cache Merging} \label{sec:app:kv_merging}

Merging methods integrate less important KV into others rather than discarding them directly. Specifically, mainstream studies select, for each KV pair to be evicted, a merging target among the preserved KVs, allowing many-to-one merges. Typically, weighted merging rather than direct averaging is used, with weights satisfying a normalization constraint, i.e., the merged vectors are obtained via convex combinations. Formally, merging the evicted pairs $(k_j, v_j), j\in K_e$ into a preserved pair $(k_c, v_c)$ yields a new KV pair $(k_r, v_r)$, defined as follows:

\begin{align}
k_r &= w_c k_c + \sum_{j \in K_e} w_j k_j, \notag \\
v_r &= w_c v_c + \sum_{j \in K_e} w_j v_j, \notag \\
\text{s.t.}&\quad w_c + \sum_{j \in K_e} w_j = 1
\label{eq:app:normal_merging}
\end{align}

Let $K'_e = K_e \cup \{c\}$, representing the index of the original KVs before merging. For instance, the weight $w_j$ in D2O \cite{wan2024D2O} is computed based on the cosine similarity between key vectors, whereas for KVmerger \cite{wang2024KVMerger}, it is calculated based on the Gaussian Kernel value. Formally, these are represented as follows:

\begin{align}
w_{j_{\text{D2O}}} &= \frac{\exp(\cos \theta_{\mathbf{k_j}, \mathbf{k_c}})}
{\sum_{j \in K'_e} \exp(\cos \theta_{\mathbf{k_j}, \mathbf{k_c}})}, \notag \\
w_{j_{\text{KVMerger}}} &= \frac{\exp\left(-\frac{||k_j - k_c||^2}{2\sigma^2}\right)}
{\sum_{j \in K'_e} \exp\left(-\frac{||k_j - k_c||^2}{2\sigma^2}\right)}.
\end{align}

However, the widely adopted convex combination approach also introduces output disturbances, as stated in the following theorem:

\begin{theorem}[Formal version of Theorem \ref{thm:attn_collapse}]
The merging method indicated by Equation \ref{eq:app:normal_merging} causes the attention score of the merged KV to become less than the sum of attention scores from the original multiple KVs merged into it, independently of the specific weighting scheme. Formally, this implies: ${A'}_r^{t} < \sum_{j\in K'_e}A^t_j$, ultimately leading to: $\left\| o'_t - o_t \right\| > 0$.
\end{theorem}

\begin{proof}
The attention score and output after merging can be expressed as:

\begin{equation}
o'_t = \sum_{i=1,i\notin K'_e}^{t}{A'}_i^tv_i + {A'}_r^tv_r = \frac{\sum_{i=1,i\notin K'_e}^{t}s^t_iv_i+ s^t_rv_r}{\sum_{i=1,i\notin K'_e}^{t}s^t_i + s^t_r}.
\label{eq:app:merging_output}
\end{equation}

First, we compare the denominators of $A^t$ and $A'^t$, formally proving $\sum_{i=1,i\notin K'_e}^{t}s^t_i + s^t_r < \sum_{i=1}^{t}s^t_i$:

\begin{equation}
\sum_{i=1,i\notin K'_e}^{t}s^t_i + s^t_r = \sum_{i=1}^{t}s^t_i + s^t_r - \sum_{i\in K'_e}s^t_i = e^{\frac{q^tk_r}{\sqrt{d}}}-\sum_{i\in K'_e}e^{\frac{q^tk_i}{\sqrt{d}}}
\end{equation}

Substituting Equation \ref{eq:app:normal_merging}, and applying the Weighted AM–GM Inequality, we have:

\begin{equation}
\begin{aligned}
e^{\frac{q^tk_r}{\sqrt{d}}}-\sum_{i\in K'_e}e^{\frac{q^tk_i}{\sqrt{d}}} &= e^{\frac{q^t(\sum_{i\in K'_e}w_ik_i)}{\sqrt{d}}}-\sum_{i\in K'_e}e^{\frac{q^tk_i}{\sqrt{d}}} \\ 
&= \prod_{i\in K'_e}{(e^{\frac{q^tk_i}{\sqrt{d}}})}^{w_i} - \sum_{i\in K'_e}e^{\frac{q^tk_i}{\sqrt{d}}} \\
&\leq \sum_{i\in K'_e}w_ie^{\frac{q^tk_i}{\sqrt{d}}} - \sum_{i\in K'_e}e^{\frac{q^tk_i}{\sqrt{d}}} < 0
\end{aligned}
\end{equation}

Thus,

\begin{equation}
\sum_{i=1,i\notin K'_e}^{t}s^t_i + s^t_r = \sum_{i=1}^{t}s^t_i + (s^t_r - \sum_{i\in K'_e}s^t_i) < \sum_{i=1}^{t}s^t_i
\end{equation}

Since the sum of the normalized attention scores equals one, and given that $\sum_{i=1,i\notin K'_e}^{t}s^t_i + s^t_r < \sum_{i=1}^{t}s^t_i$, we obtain:

\begin{align}
{A'}_r^t &= 1 - \frac{\sum_{i=1,\,i\notin K'_e}^{t}s^t_i }
                 {\sum_{i=1,\,i\notin K'_e}^{t}s^t_i + s^t_r} \notag \\
&< 1 - \frac{\sum_{i=1,\,i\notin K'_e}^{t}s^t_i}
             {\sum_{i=1}^{t}s^t_i} = \sum_{i\in K'_e}{A}_i^t
\end{align}

Similarly, we can derive:

\begin{equation}
{A'}_j^t = \frac{s^t_j}{\sum_{i=1,i\notin K'_e}^{t}s^t_i + s^t_r} > \frac{s^t_j}{\sum_{i=1}^{t}s^t_i} = {A}_j^t, \quad j\neq r
\end{equation}

Finally, the output perturbation can be represented as:

\begin{equation}
\begin{aligned}
\left\|o'_t-o_t\right\| &= \left\|(\sum_{i=1,i\notin K'_e}^{t}A_i^{'t}v_i + {A'}_r^tv_r) - \sum_{i=1}^{t}A_i^tv_i\right\| \\
&= \left\|\sum_{i=1,i\notin K'_e}^{t}(A_i^{'t}-A_i^t)v_i + ({A'}_r^tv_r - \sum_{i\in K'_e}{A}_i^tv_i)\right\| \\
&= \left\|\sum_{i=1,i\notin K'_e}^{t}(A_i^{'t}-A_i^t)v_i + \sum_{i\in K'_e}(w_i{A'}_r^t - A_i^t)v_i\right\| 
\end{aligned}
\end{equation}

In the above expression, all vector coefficients are nonzero. Moreover, due to the high dimensionality and sparsity of the KV cache \cite{wang2024KVMerger,gu2025attentionsink_empiricalview}, the vectors are almost linearly independent. In practical inference scenarios, it is impossible for them to form a zero vector through linear combination. Consequently, we have:
$\left\| o'_t - o_t \right\| > 0$.
\end{proof}

We term this phenomenon as \textbf{\MergingProblemName{}}, indicating that improper merging methods result in a reduced attention score for the newly merged vector, while attention scores of unmerged KVs relatively increase. This leads to output disturbances and ultimately degrades generation quality.


\subsection{KeepKV Merging Method} \label{sec:app:keepkv_method}

In the main text, we introduced the concept of merging count via the Electoral Votes mechanism, aiming for a KV pair with vote count $p_i$ to be equivalent, in attention computation, to $p_i$ independent occurrences of this KV. Moreover, the vote count of the merged KV equals the sum of vote counts before merging. Formally, the outputs before $(o_t)$ and after merging $(o'_t)$ can be expressed as follows:

\begin{align}
o_{t} &= \frac{\sum_{i=1}^{t} p_i s_i^{t} v_i}{\sum_{i=1}^{t} p_i s_i^{t}}, \notag \\
o'_{t} &= \frac{\sum_{i=1,\, i \notin K'_e}^{t} p_i s_i^{t} v_i + p_r s_r^{t} v_r}
{\sum_{i=1,\, i \notin K'_e}^{t} p_i s_i^{t} + p_r s_r^{t}}, \notag \\
p_r &= \sum_{i \in K'_e} p_i.
\label{eq:app:keepkv_attention_compute}
\end{align}

Next, we demonstrate how our new merging approach can be derived naturally from the objective of eliminating output disturbances, which consequently serves as a direct proof for Theorem \ref{thm:keepkv_zero_perturbation}.

Based on Equation \ref{eq:app:keepkv_attention_compute}, setting $\left\| o'_t - o_t \right\| = 0$, we obtain:

\begin{align}
\sum_{i=1}^{t} p_i s_i^{t} v_i &= \sum_{i=1,\, i \notin K'_e}^{t} p_i s_i^{t} v_i + p_r s_r^{t} v_r, \notag \\
\sum_{i=1}^{t} p_i s_i^{t} &= \sum_{i=1,\, i \notin K'_e}^{t} p_i s_i^{t} + p_r s_r^{t}.
\end{align}

which implies:

\begin{equation}
\sum_{i\in K'_e}p_is_i^{t}v_i = p_rs_r^{t}v_r, \quad \sum_{i\in K'_e}p_is_i^{t} =  p_rs_r^{t}
\label{eq:app:o=o,o=o}
\end{equation}

Dividing the two expressions above, we obtain the expression of $v_r$:

\begin{equation}
v_r = \frac{\sum_{i\in K'_e}p_is_i^{t}v_i}{\sum_{i\in K'_e}p_is_i^{t}} 
\end{equation}

Similarly, let $k_r = C(\sum_{i\in K'_e} p_is_i^tk_i)$, substituting this into $\sum_{i\in K'_e}^{t}p_is_i^{t} =  p_rs_r^{t}$ from Equation \ref{eq:app:o=o,o=o} and solving, we obtain:

\begin{equation}
C = \frac{\ln{\frac{\sum_{i\in K'_e}p_is_i^t}{\sum_{i\in K'_e}p_i}}}{\sum_{i\in K'_e}p_is_i^t\ln{s_i^t}}
\end{equation}

Finally, we derive the merging expression:

\begin{align}
k_r &= \frac{\left( \sum_{i \in K'_e} p_i s_i^{t} k_i \right) \ln{\left( \frac{\sum_{i \in K'_e} p_i s_i^t}{\sum_{i \in K'_e} p_i} \right)}}
{\sum_{i \in K'_e} p_i s_i^t \ln{s_i^t}}, \notag \\
v_r &= \frac{\sum_{i \in K'_e} p_i s_i^{t} v_i}{\sum_{i \in K'_e} p_i s_i^{t}}, \quad
p_r = \sum_{i \in K'_e} p_i.
\label{eq:app:keepkv_merge}
\end{align}

Consequently, merging in this manner eliminates the output disturbance in the $t$-step, satisfying: $\left\| o'_t - o_t \right\| = 0$. By setting the merging candidates $K'_e=\{e\}\cup\{c\}$, we obtain \textbf{Theorem \ref{thm:keepkv_zero_perturbation}}.

\subsection{Error Bound Analysis} \label{sec:app:error_analysis}

After extending KeepKV to multi-step generation, for $t'$-step, all $s_i^t$ terms in Equation \ref{eq:app:keepkv_merge} are replaced with $\hat{s_i^{t'}}$ 
 , which represents our estimation of future attention score trends obtained through a certain method. In this case, the merging expressions become:

\begin{align}
k_r &= \frac{\left( \sum_{i \in K'_e} p_i \hat{s}_i^{t'} k_i \right) \ln{\left( \frac{\sum_{i \in K'_e} p_i \hat{s}_i^{t'}}{\sum_{i \in K'_e} p_i} \right)}}
{\sum_{i \in K'_e} p_i \hat{s}_i^{t'} \ln{\hat{s}_i^{t'}}}, \notag \\
v_r &= \frac{\sum_{i \in K'_e} p_i \hat{s}_i^{t'} v_i}{\sum_{i \in K'_e} p_i \hat{s}_i^{t'}}, \quad
p_r = \sum_{i \in K'_e} p_i.
\label{eq:app:keepkv_merge_predict}
\end{align}

and they satisfy:

\begin{equation}
\sum_{i\in K'_e}p_i\hat{s_i^{t'}}v_i = p_r\hat{s_r^{t}}v_r, \quad \sum_{i\in K'_e}p_i\hat{s_i^{t}} =  p_r\hat{s_r^{t}}
\label{eq:app:prediction_o=o,o=o}
\end{equation}

Then the perturbation at step $s'$, can be expressed as:

\begin{align}
&\Theta_{t'} = \left\| o_{t'} - o'_{t'} \right\| \notag \\
&= \left\| 
\frac{\sum_{i=1}^{t'} p_i s_i^{t'} v_i}{
      \sum_{i=1}^{t'} p_i s_i^{t'}}
-
\frac{
\sum_{i=1,\,i\notin K'_e}^{t'} p_i s_i^{t'} v_i 
+ p_r s_r^{t'} v_r
}{
\sum_{i=1,\,i\notin K'_e}^{t'} p_i s_i^{t'} 
+ p_r s_r^{t'}
}
\right\| \notag \\
&= \frac{
\left\| \sum_{i=1}^{t'} p_i s_i^{t'} \left(
\sum_{j\in K'_e} p_j s_j^{t'} (v_j - v_i)
- p_r s_r^{t'} (v_r - v_i)
\right) \right\|
}{
\left( \sum_{i=1}^{t'} p_i s_i^{t'} \right)
\left( \sum_{i=1,\,i\notin K'_e}^{t'} p_i s_i^{t'} 
+ p_r s_r^{t'} \right)
}
\label{eq:app:theta}
\end{align}

Substituting the expression for $v_r$ from Equation \ref{eq:app:keepkv_merge_predict} and $\sum_{j\in K'_e}p_j\hat{s_j^{t'}} = p_r\hat{s_r^{t'}}$ from Equation \ref{eq:app:prediction_o=o,o=o} into the above:

\begin{align}
&\sum_{i=1}^{t'} p_i s_i^{t'} \left[
    \sum_{j \in K'_e} p_j s_j^{t'} (v_j - v_i)
    - p_r s_r^{t'} \left(
        \frac{\sum_{k \in K'_e} p_k \hat{s}_k^{t'} v_k}{
              \sum_{k \in K'_e} p_k \hat{s}_k^{t'} }
        - v_i
    \right)
\right] \notag \\
&= \sum_{i=1}^{t'} p_i s_i^{t'} \sum_{j \in K'_e} p_j s_j^{t'}
    \left( 1 - \frac{s_r^{t'}}{\hat{s}_r^{t'}} \cdot
                \frac{\hat{s}_j^{t'}}{s_j^{t'}} \right)
    (v_j - v_i)
\label{eq:app:RE_fenzi}
\end{align}

Let $\left| 1-\frac{\hat{s_i^{t'}}}{s_i^{t'}} \right| \leq \epsilon, \epsilon < 1$, then $1-\epsilon\leq\frac{\hat{s_i^{t'}}}{s_i^{t'}}\leq 1+\epsilon $, thus:

\begin{equation}
\begin{aligned}
\left|1-\frac{s_r^{t'}}{\hat{s_r^{t'}}}\frac{\hat{s_j^{t'}}}{s_j^{t'}} \right| = \left|\frac{\frac{\hat{s_r^{t'}}}{s_r^{t'}}-\frac{\hat{s_j^{t'}}}{s_j^{t'}}}{\frac{\hat{s_r^{t'}}}{s_r^{t'}}} \right| \leq \frac{2\epsilon}{1-\epsilon}, \quad j\in K'_e
\end{aligned}
\end{equation}

Let $\left\|v_j-v_i\right\|\leq\gamma,\forall i\in [t'], j\in K'_e$, where $\gamma$ represents the inherent variation in the input, which cannot be eliminated through algorithmic design. Then, applying the triangle inequality, we obtain:

\begin{align}
&\left\| \sum_{i=1}^{t'} p_i s_i^{t'} \left[
\sum_{j\in K'_e} p_j s_j^{t'} \left(1 - \frac{s_r^{t'}}{\hat{s_r^{t'}}}
\cdot \frac{\hat{s_j^{t'}}}{s_j^{t'}} \right)(v_j - v_i)
\right] \right\| \notag \\
&\leq \sum_{i=1}^{t'} p_i s_i^{t'} 
\left( \sum_{j\in K'_e} p_j s_j^{t'} 
\left| 1 - \frac{s_r^{t'}}{\hat{s_r^{t'}}}
\cdot \frac{\hat{s_j^{t'}}}{s_j^{t'}} \right| \cdot \| v_j - v_i \| \right) \notag \\
&\leq \frac{2 \epsilon \gamma}{1 - \epsilon} 
\left( \sum_{i=1}^{t'} p_i s_i^{t'} \right)
\left( \sum_{j\in K'_e} p_j s_j^{t'} \right)
\label{eq:app:bound_norm}
\end{align}

Substituting this inequality into Equation \ref{eq:app:theta}, we obtain:

\begin{equation}
\begin{aligned}
\Theta_{t'} &\leq \frac{2\epsilon\gamma}{1-\epsilon} 
\frac{(\sum_{i=1}^{t'}p_is_i^{t'})(\sum_{j\in K'_e}p_js_j^{t'})}{(\sum_{i=1}^{t'}p_is_i^{t'})(\sum_{i=1,i\notin K'_e}^{t'}p_is_i^{t'} + p_rs_r^{t'})} \\
&= \frac{2\epsilon\gamma}{1-\epsilon}
\frac{\sum_{j\in K'_e}p_js_j^{t'}}{\sum_{i=1,i\notin K'_e}^{t'}p_is_i^{t'} + p_rs_r^{t'}} \\ 
&< \frac{2\epsilon\gamma}{1-\epsilon}
\frac{\sum_{j\in K'_e}p_js_j^{t'}}{p_rs_r^{t'}}
\end{aligned}
\end{equation}

Due to Equation \ref{eq:app:prediction_o=o,o=o}, we have $\frac{\sum_{j\in K'_e}p_j\hat{s_j^{t'}}}{p_r\hat{s_r^{t'}}}=1$, then:  

\begin{equation}
\begin{aligned}
\frac{\sum_{j\in K'_e}p_js_j^{t'}}{p_rs_r^{t'}} \leq \frac{\frac{1}{1-\epsilon}\sum_{j\in K'_e}p_j\hat{s_j^{t'}}}{\frac{1}{1+\epsilon}p_r\hat{s_r^{t'}}} = \frac{1+\epsilon}{1-\epsilon}
\end{aligned}
\end{equation}

Thus, 

\begin{equation}
\begin{aligned}
\Theta_{t'} &< \frac{2\epsilon\gamma}{1-\epsilon}
\frac{\sum_{j\in K'_e}p_js_j^{t'}}{p_rs_r^{t'}} < \frac{2\epsilon (1+\epsilon)\gamma}{(1-\epsilon)^2} 
\end{aligned}
\end{equation}

Finally, we obtain the following theorem:

\begin{theorem} 
For the $t'$-th step, let $\left|1-\frac{\hat{s_i^{t'}}}{s_i^{t'}} \right| \leq \epsilon, \epsilon < 1$, the output perturbation satisfies $\Theta_{t'} < \frac{2\epsilon (1+\epsilon)\gamma}{(1-\epsilon)^2}  $, provided that $\left\|v_j-v_i\right\| \leq \gamma,\forall i\in [t'], j\in K'_e$.
\label{thm:app:KeepKV_RE}
\end{theorem}

Next, we prove the following lemma:

\begin{lemma}
As the prediction error $\epsilon$ decreases and the merged candidates become increasingly similar, the output disturbance reduces to zero. That is, when either $\epsilon=0$ or $ (k_i,v_i)=(k_j,v_j),\forall i,j\in K'_e$, we have: $\Theta_{t'} = 0$.
\label{lemma:app:RE_to0}
\end{lemma}

\begin{proof}

By Theorem \ref{thm:app:KeepKV_RE}, it is easy to obtain that when $\epsilon = 0$, $\Theta_{t'} < \frac{2\epsilon (1+\epsilon)\gamma}{(1-\epsilon)^2}=0$. Next, we prove that when $(k_i,v_i) = (k_j,v_j), \forall i,j \in K'_e$, it also holds that $\Theta_{t'} = 0$. First, we further expand $(1-\frac{s_r^{t'}}{\hat{s_r^{t'}}}\frac{\hat{s_j^{t'}}}{s_j^{t'}}), j\in K'_e$ in Equation \ref{eq:app:RE_fenzi} by applying Equation \ref{eq:app:prediction_o=o,o=o}:

\begin{equation}
\begin{aligned}
1-\frac{s_r^{t'}}{\hat{s_r^{t'}}}\frac{\hat{s_j^{t'}}}{s_j^{t'}} &= 1- \frac{e^{\frac{q_{t'}k_r}{\sqrt{d}}}}{\frac{\sum_{i\in K'_e}p_i\hat{s_i^{t'}}}{\sum_{i\in K'_e}p_i}}\frac{\hat{s_j^{t'}}}{s_j^{t'}},\quad j\in K'_e 
\end{aligned}
\end{equation}

Substituting the expression for $k_r$ from Equation \ref{eq:app:keepkv_merge_predict} into the above:

\begin{align}
&1 - \frac{e^{\frac{q_{t'} k_r}{\sqrt{d}}}}{
\frac{\sum_{i \in K'_e} p_i \hat{s}_i^{t'}}{
\sum_{i \in K'_e} p_i
}
} \cdot \frac{\hat{s}_j^{t'}}{s_j^{t'}} \notag \\
&= 1 - 
\frac{
\left( \prod_{i \in K'_e} (s_i^{t'})^{p_i \hat{s}_i^{t'}} \right)^{
\frac{
\ln \left( \frac{\sum_{i \in K'_e} p_i \hat{s}_i^{t'}}{
\sum_{i \in K'_e} p_i} \right)
}{
\sum_{i \in K'_e} p_i \hat{s}_i^{t'} \ln \hat{s}_i^{t'}
}
}
}{
\frac{\sum_{i \in K'_e} p_i \hat{s}_i^{t'}}{
\sum_{i \in K'_e} p_i
}
} \cdot \frac{\hat{s}_j^{t'}}{s_j^{t'}}
\label{eq:app:exp_error_bound}
\end{align}

When $(k_i,v_i) = (k_j,v_j), \forall i,j \in K'_e$, it follows that $\forall i \in K'_e, s_i^{t'}=s^{t'}, \hat{s_i^{t'}}=\hat{s^{t'}}$, thereby:

\begin{align}
&1 - 
\frac{
\left( \prod_{i \in K'_e} {s_i^{t'}}^{p_i \hat{s}_i^{t'}} \right)^{
\frac{
\ln \left( \frac{ \sum_{i \in K'_e} p_i \hat{s}_i^{t'} }{ \sum_{i \in K'_e} p_i } \right)
}{
\sum_{i \in K'_e} p_i \hat{s}_i^{t'} \ln \hat{s}_i^{t'}
}
}
}{
\frac{ \sum_{i \in K'_e} p_i \hat{s}_i^{t'} }{ \sum_{i \in K'_e} p_i }
} \cdot \frac{ \hat{s}_j^{t'} }{ s_j^{t'} } \notag \\
&= 1 - 
\frac{
\left( \prod_{i \in K'_e} {s^{t'}}^{p_i \hat{s}^{t'}} \right)^{
\frac{
\ln \left( \frac{ \sum_{i \in K'_e} p_i \hat{s}^{t'} }{ \sum_{i \in K'_e} p_i } \right)
}{
\sum_{i \in K'_e} p_i \hat{s}^{t'} \ln \hat{s}^{t'}
}
}
}{
\frac{ \sum_{i \in K'_e} p_i \hat{s}^{t'} }{ \sum_{i \in K'_e} p_i }
} \cdot \frac{ \hat{s}^{t'} }{ s^{t'} } \notag \\
&= 1 - \frac{ s^{t'} }{ \hat{s}^{t'} } \cdot \frac{ \hat{s}^{t'} }{ s^{t'} } = 0
\label{eq:app:zero_proof}
\end{align}

Under this condition, it follows that Equation \ref{eq:app:RE_fenzi} equals 0, i.e.,

\begin{align}
&\sum_{i=1}^{t'} p_i s_i^{t'} \left[
\sum_{j \in K'_e} p_j s_j^{t'} (v_j - v_i)
- p_r s_r^{t'} (v_r - v_i)
\right] \notag \\
&= \sum_{i=1}^{t'} p_i s_i^{t'} \left[
\sum_{j \in K'_e} p_j s_j^{t'} 
\left(1 - \frac{s_r^{t'}}{\hat{s}_r^{t'}} \cdot \frac{\hat{s}_j^{t'}}{s_j^{t'}} \right) (v_j - v_i)
\right] = 0
\label{eq:app:perturbation_zero_condition}
\end{align}

Finally, Substituting it into Equation \ref{eq:app:theta}, we obtain $\Theta_{t'} = 0$.
\end{proof}

\begin{remark}
This lemma provides a theoretical justification for prior merging strategies favoring high-similarity KV pairs \cite{wan2024D2O,wang2024KVMerger}. Meanwhile, we offer an intuitive interpretation: if the merged two KV pairs are identical, i.e., $(k_e, v_e) = (k_c, v_c)$, retaining one pair and setting its Electoral Votes to 2 introduces no error in subsequent computations.
\end{remark}

In this section, we have proven Theorem \ref{thm:app:KeepKV_RE} and Lemma \ref{lemma:app:RE_to0}, which provide guarantees on the output perturbation in multi-step generation—an aspect that existing methods struggle to achieve. Moreover, our method demonstrates superior performance across various experimental evaluations. However, it should be acknowledged that predicting attention distributions further into the future is inherently challenging, leading to a significant increase in the estimated perturbation upper bound. Furthermore, the inherent input differences $\gamma$ cannot be ignored, representing a fundamental problem in KV cache compression—namely, the inability to perfectly compress the KV cache into a smaller memory without any loss of information. Nevertheless, our work introduces a new perspective and analytical approach to studying KV cache eviction and merging algorithms, which we hope will inspire future research.

\section{Implementation Details} \label{sec:app:impl}

\subsection{Models and benchmarks}

Across all experiments, we utilized pre-trained model weights from Huggingface \cite{wolf2020huggingfacestransformers}. Specifically, for Llama architectures: we used the 'huggyllama/llama-7b' checkpoint for Llama-1-7B \cite{touvron2023llama}, 'meta-llama/Llama-2-7b-hf' for Llama-2-7B \cite{touvron2023llama2}, and 'meta-llama/Llama-2-13b-hf' for Llama-2-13B. For Llama-3-8B, the 'meta-llama/Meta-Llama-3-8B-Instruct' checkpoint was used in the HELM evaluations, and 'meta-llama/Meta-Llama-3-8B' in LongBench evaluations. Regarding the Mistral architecture, we employed the 'mistralai/Mistral-7B-Instruct-v0.3' model \cite{jiang2023mistral7b}. For detailed information on the LongBench benchmark, please refer to its official repository \cite{bai2024longbench}.

\subsection{Experimental Setup}

Our study does not involve model training, thus no data preprocessing is required. All evaluation datasets are sourced from the publicly available lm-eval-harness \cite{lmevalharness}, HELM \cite{narayan2018xsum,nallapati2016cnndm} and LongBench \cite{bai2024longbench} benchmarks, and we follow their original evaluation metrics. Our implementation is primarily based on modifications to the open-source H2O and D2O codebases \cite{zhang2023h2o,wan2024D2O}, both of which are publicly accessible. All open-source datasets, evaluation frameworks, and algorithmic implementations used in this work are employed in full compliance with their respective licenses and terms of use.

\subsection{Details of Parameter Settings}

For cache allocation, we adopt the default configuration provided by the open-source PyramidInfer \cite{yang2024pyramidinfer} implementation. Specifically, we follow a cosine-based decay strategy for distributing cache across layers. Within each layer, the ratio between the recent window and the heavy hitter (i.e., crucial KV entries) varies between 5:1 and 4:1. Additionally, we retain the first 4 tokens (referred to as 'sink tokens' in StreamingLLM \cite{xiao2024StreamingLLM}) in the cache throughout. For algorithmic hyperparameters, we adjust the merging threshold $T$ within the range of 0.6 to 0.95, and empirically select 0.8 as the default value based on its stable performance across benchmarks. The corresponding experimental results are reported below.

\subsubsection{Prefill's Merge Policy}

\begin{table}[htbp]
\centering
\setlength{\tabcolsep}{4pt} 
\begin{tabular}{c|c|c}
\hline
Cache Budget & 1\% & 10\% \\
\hline
Evict-first & 0.051 & 0.111 \\
\hline
Merge-first & 0.059 & 0.115 \\
\hline
\end{tabular}
\caption{Performance comparison
with different merge policy during the prefill stage (ROUGE-2; LLAMA-7B).}
\label{tab:app:merge_policy}
\end{table}

The results in Table~\ref{tab:app:merge_policy} show that, during the prefill stage, performing merging before applying the eviction strategy—rather than merging after eviction as is commonly done—can improve generation quality. This is because, according to Lemma~\ref{lemma:RE_towards_0}, merging more similar items induces smaller perturbations. In the prefill stage, there exist more tokens whose similarity exceeds the threshold; merging them first allows these tokens to be consolidated rather than evicted unnecessarily, thereby preserving more information within the limited-size cache.

\subsubsection{Merging Threshold}

\begin{table}[htbp]
\centering
\setlength{\tabcolsep}{4pt} 
\begin{tabular}{c|c}
\hline
Merging Thres & XSUM \\
\hline
0.7 & 0.223 \\
\hline
0.8 & 0.233 \\
\hline
0.9 & 0.214 \\
\hline
\end{tabular}
\caption{Merging threshold impact (4\% compression ratio).}
\label{tab:app:merge_thres}
\end{table}

We investigate the impact of different merging thresholds on performance. The results in Table~\ref{tab:app:merge_thres} indicate that setting the merging threshold (i.e., the cosine similarity between key vectors) to 0.8 yields better performance, and its robustness has been further validated across a broader range of experiments.

\section{Additional Discussion}

\subsection{Limitations and Future Work}

As analyzed in Section~\ref{sec:app:thero}, the bound on output perturbation in multi-step generation depends on two quantities: the intrinsic difference of the inputs $\gamma$, and the prediction error $\epsilon$. The former is inherent and cannot be influenced or altered, whereas the latter is inevitably non-negligible in realistic inference scenarios. In practice, accurately predicting future attention scores is highly challenging. Prior work typically employs sliding-window averages or weighted averages as predictors, while more recent studies have begun to explore machine-learning-based approaches to obtain more accurate predictions; this appears to be a promising direction for future research.

Moreover, the merging operation inevitably incurs additional computational overhead, which can render it less favorable than purely eviction-based strategies with simpler computation in certain settings. Our mathematically derived lossless KV cache merging method, along with many other related algorithms, further relies on access to attention scores, which limits its compatibility with widely adopted inference acceleration techniques such as FlashAttention. Future work could therefore focus on designing KV cache eviction and merging algorithms that are amenable to integration with FlashAttention and related acceleration methods.

\subsection{Comparison with CaM in Eliminating Output Perturbations}

A closely related line of work is CaM \cite{zhang2024cam}, which also analyzes the effect of KV cache eviction and its proposed merging algorithm on output perturbation, and aims to reduce or even eliminate such perturbation. However, there are two fundamental differences between CaM and our approach: (i) how "output perturbation" is defined and modeled, and (ii) how the theoretical analysis is connected to the asynchronous KV cache updates that actually occur in autoregressive decoding.

Before presenting our analysis, it is important to clarify the implementation setting shared by both CaM and our method. In practical autoregressive decoding, KV compression is always performed inside the decoding loop: the model first runs the forward pass at step $t$ using the current KV cache and produces the output, and only afterwards applies eviction or merging operations to the cache state that will be consumed at step $t+1$ and beyond. In other words, cache compression is asynchronous with respect to the forward computation. Consequently, when we design an algorithm from the perspective of eliminating output perturbation, our first requirement is consistency at the current step: if we hypothetically recompute the step-$t$ forward pass on the compressed KV cache, the resulting output should match that obtained from the original, uncompressed cache. Intuitively, the attention computation between the step-$t$ query and the pre-compression KV cache should yield exactly the same result as the attention between the same query and the post-compression KV cache. This ensures that, when we "look back" at step $t$ in the future, its behavior is indistinguishable from a world in which step $t$ had always been computed on the compressed KV cache, without needing to refer to the original KV entries that are no longer stored. The overall decoding process thus appears temporally coherent.

In contrast, when CaM studies the effect of eviction and merging on the output at step $t$, it does not adopt such a "recompute on the compressed cache" viewpoint. Instead, it directly removes (or rescales) the contribution of the evicted or merged KV pair in the attention scores and in the attention-weighted output, while leaving the contributions of all remaining tokens unchanged. This effectively treats the attention scores of the other KV pairs as fixed coefficients that are not affected by modifying the cache. However, due to the softmax normalization, evicting or merging any KV pair inevitably changes the normalization term and thus the attention weights of all remaining KV pairs. Capturing this global effect of softmax is precisely why our analysis starts from the full attention expression, rather than from a model in which the other attention weights are held constant. In summary, CaM and our method formalize and analyze output perturbation from different perspectives, and each seeks to reduce or eliminate perturbation under its own definition, but our framework is explicitly aligned with the asynchronous KV-update pattern and with the impact of softmax on all tokens in the cache.

\end{document}